\documentclass[11pt]{article}

\usepackage[margin=1in]{geometry}
\usepackage{times}
\usepackage[round,authoryear]{natbib}
\usepackage{amsmath,amsfonts,bm}

\def\eqref#1{equation~\ref{#1}}
\def\1{\bm{1}}

\DeclareMathAlphabet{\mathsfit}{\encodingdefault}{\sfdefault}{m}{sl}
\SetMathAlphabet{\mathsfit}{bold}{\encodingdefault}{\sfdefault}{bx}{n}

\usepackage{amsmath}
\usepackage{amssymb}
\usepackage{amsthm}
\usepackage{booktabs}
\usepackage{graphicx}
\usepackage{algorithm}
\usepackage{algpseudocode}
\usepackage{placeins}
\usepackage{flafter}
\usepackage{hyperref}
\usepackage{url}
\numberwithin{equation}{section}
\hypersetup{
  hidelinks,
  pdftitle={FREESIA: Covariance-Aware Posterior Transport for Expressive and Scalable Data Assimilation},
  pdfauthor={Shiwei Ni; Yangwen Zhang; Hang Qi; Xiaofei Guan; Lili Ju}
}
\title{FREESIA: Covariance-Aware Posterior Transport for\\
Expressive and Scalable Data Assimilation}
\author{%
\href{https://orcid.org/0009-0004-5625-1682}{Shiwei Ni}\textsuperscript{1}
\quad
\href{https://orcid.org/0009-0002-1502-8332}{Yangwen Zhang}\textsuperscript{1}
\quad
\href{https://orcid.org/0009-0007-2727-2734}{Hang Qi}\textsuperscript{1}
\quad
\href{https://orcid.org/0000-0002-2062-2486}{Xiaofei Guan}\textsuperscript{1}
\quad
\href{https://people.math.sc.edu/ju/}{Lili Ju}\textsuperscript{2}\\[0.5em]
\small \textsuperscript{1}School of Mathematical Sciences, Tongji University,
Shanghai, China\\
\small \textsuperscript{2}Department of Mathematics, University of South Carolina,
Columbia, SC 29208, USA\\
\small Email: \texttt{2433953@tongji.edu.cn}}
\date{}

\newcommand{\method}{FREESIA}

\newtheorem{theorem}{Theorem}
\newtheorem{proposition}{Proposition}
\newtheorem{corollary}{Corollary}
\newtheorem{lemma}{Lemma}

\begin{document}

\maketitle

\begin{abstract}
Data assimilation aims to infer the state of complex dynamical systems based on
observational data. However, accurate inference of the multimodal posteriors
induced by nonlinear or non-injective observation operators remains a key
challenge under high-dimensional and sparse observation conditions. Ensemble
filters scale to high dimensions but are confined by restrictive distributional
assumptions, while training-free generative filters (e.g., EnSF, EnFF) alleviate
this limitation but may introduce structural errors and hinder information
propagation under sparse observations. To address these issues, we propose a
training-free, asymptotically exact posterior transport method. Firstly, a
covariance-aware posterior transport scheme is designed, which embeds the
forecast cross-covariance into flow-based transport and accurately recovers
unobserved states while preserving the non-Gaussian posterior structure.
Furthermore, the method combines a tractable observation-adaptive proposal with
posterior correction, ensuring accurate approximation of the nonlinear
posterior distribution. Finally, we establish the corresponding posterior flow
theory, from which the asymptotic exactness of the proposed method relative to
finite-ensemble surrogates and the Wasserstein error bound are derived.
Experiments on Double-Well, Lorenz--96, and Kolmogorov flow show that the
proposed method captures complex posterior structure and remains accurate under
sparse, nonlinear, and non-injective observations. In the sparse non-injective
setting, it reduces RMSE by 56\% relative to the best baseline.

\end{abstract}

\section{Introduction}

Data assimilation (DA) combines dynamical-model forecasts with observations to
estimate the evolving state of a dynamical system
and quantify its uncertainty, with important applications in weather
forecasting, climate modeling, and other areas of science and engineering
\citep{law2015dataassimilation}. In Bayesian filtering, the state estimate and
its uncertainty are represented by the filtering distribution, which
characterizes the state conditioned on the available observations and is
updated sequentially as new data become available.
In high-dimensional nonlinear systems, sparse observations provide only
partial information about the state, while nonlinear or non-injective
observation operators can produce non-Gaussian or multimodal filtering
distributions. Effective filtering must therefore both propagate observational
information to unobserved variables and represent complex filtering
distributions.

Classical ensemble Kalman methods such as the EnKF and LETKF propagate
observational information across the state through forecast covariances, using
localization and inflation to remain effective with finite ensembles in high
dimensions \citep{evensen2003enkf,hunt2007letkf,gaspari1999localization}.
Iterative and tempered variants further improve robustness to nonlinear
observations through repeated updates \citep{emerick2012enkfmda}. However,
these methods rely on near-Gaussian approximations of the filtering
distribution, which can be restrictive under strongly nonlinear or
non-injective observations. By contrast, particle filters can represent
non-Gaussian and
multimodal filtering distributions, but suffer severe weight degeneracy in
high dimensions \citep{gordon1993particle,snyder2008particle}.

These limitations have motivated score-based generative approaches for
representing complex, high-dimensional, non-Gaussian filtering distributions.
Score-based Data Assimilation learns a generative score model of dynamical
trajectories and incorporates observations through guidance during assimilation
\citep{rozet2023scorebased}, whereas SSLS recursively learns the forecast score
at each assimilation step \citep{ding2026ssls}. Related learned generative
filters include flow-based Bayesian filtering, which uses normalizing flows to
construct a latent linear-Gaussian state-space model for efficient
high-dimensional nonlinear filtering \citep{wang2025flowbased}. These
approaches therefore rely on learned generative representations, either trained
offline or re-estimated from data. To avoid such training dependence, the
Ensemble Score Filter (EnSF) estimates the forecast score directly from the
current ensemble and incorporates observational information through likelihood
guidance during reverse diffusion \citep{bao2024ensf}. This training-free
score-filtering framework has since been extended to geophysical turbulence,
multiphase-flow state estimation, adaptive learning of stochastic PDE
solutions, and data-driven dynamical-system forecasting
\citep{bao2025geophysical,hu2026twophase,huynh2026spde,tang2026forecasting}.

However, EnSF's heuristic posterior-score construction introduces structural
error, while under sparse component-wise observations the likelihood gradient
acts directly only on observed variables. Several recent extensions have sought
to address these limitations from different directions. IEnSF derives a
covariance-aware posterior-score formulation under a Gaussian-mixture prior and
iteratively refines its approximation to reduce structural error
\citep{zhang2027iensf}. For partial observations, EnSF with image inpainting
reconstructs unobserved states after the score-based update
\citep{liang2025inpainting}, while EnSF-LR transfers observed-state analysis
increments to unobserved variables through prior-covariance linear regression
\citep{xiong2026ensflr}. Latent-EnSF instead addresses sparse observations
through learned coupled latent representations, but requires pretrained
encoder--decoder models and performs assimilation in the learned latent space
\citep{si2025latentensf}.

Beyond score-based diffusion, flow matching provides an ODE-based alternative
for generative transport
\citep{lipman2023flowmatching,tong2024cfm,feng2025guidance}. Its flexible design
of probability paths and couplings can yield straighter trajectories, allowing
accurate sampling with fewer integration steps. The Ensemble Flow Filter (EnFF)
adapts this framework to
training-free data assimilation and introduces a filtering-to-predictive flow
that exploits the sequential Bayesian structure of DA, improving sampling
efficiency and robustness \citep{transue2025enff}. EnFF incorporates
observations through a guidance mechanism, but its practical localized
guidance simplifies the covariance-dependent matrix structure to a tuned
scalar multiple of the identity. This simplification limits the explicit use
of cross-covariances for propagating sparse observational information to
unobserved variables.

Taken together, these developments motivate a covariance-aware, training-free
posterior-flow formulation for nonlinear data assimilation. We therefore
propose FREESIA (Flow-based Reweighted Ensemble Endpoint Sampling for Inference
and Assimilation), which constructs the posterior flow directly from the
current forecast ensemble. FREESIA embeds forecast cross-covariances into the
transport, allowing sparse observational information to propagate to
unobserved variables while retaining non-Gaussian posterior structure. To
realize this construction, the forecast distribution is represented by a
Gaussian mixture model (GMM) surrogate whose covariance structure is retained in the
conditional endpoint distribution. For nonlinear observation operators, local
linearization is used only to construct tractable endpoint proposals, while
likelihood-ratio reweighting preserves the intended nonlinear Bayesian target
under this surrogate.

Our main contributions are:
\begingroup
\setlength{\topsep}{3pt}
\setlength{\partopsep}{0pt}
\begin{itemize}
    \setlength{\itemsep}{2pt}
    \setlength{\parsep}{0pt}
    \setlength{\parskip}{0pt}
    \item We develop covariance-aware posterior transport for
    training-free nonlinear data assimilation, embedding forecast
    cross-covariances into the transport to enable accurate recovery of
    unobserved states while retaining non-Gaussian posterior structure.
    \item We introduce a proposal-and-correction strategy that combines
    tractable observation-adaptive proposals with posterior correction,
    maintaining fidelity to the nonlinear posterior target.
    \item We establish theoretical guarantees for the resulting
    posterior flow, proving asymptotic exactness with respect to the
    finite-ensemble surrogate and deriving a one-step Wasserstein error bound
    that separates finite-sampling, flow-discretization, and forecast-surrogate
    errors.
\end{itemize}
\endgroup

\noindent We evaluate FREESIA on Double-Well, Lorenz--96, and
Kolmogorov flow, covering multimodal posterior approximation, sparse
information transfer, and non-injective observations. We begin in Section~2 by
formulating the sequential Bayesian filtering problem and the finite-ensemble
surrogate underlying FREESIA. Section~3 then develops the covariance-aware
posterior transport and its proposal-and-correction construction, followed by
the theoretical analysis in Section~4. We evaluate the resulting method on
increasingly challenging nonlinear data-assimilation problems in Section~5 and
conclude in Section~6.

\section{Problem Formulation}
\label{sec:problem}

This section defines the sequential Bayesian filtering problem and the
finite-ensemble surrogate targeted by \method{}. We retain only the notation
needed to distinguish the underlying forecast/analysis laws from the
ensemble-based surrogate used to construct the posterior flow in
Section~\ref{sec:method}.

At assimilation times \(k=1,\ldots,K\), the latent states
\(X_k\in\mathbb{R}^d\) and observations \(Y_k\in\mathbb{R}^{m_k}\) satisfy
\begin{equation}
X_k = \psi_{k-1}(X_{k-1}) + \omega_k,
\qquad
Y_k = h_k(X_k) + \eta_k,
\label{eq:state_obs_model}
\end{equation}
where \(X_0\sim p_0\), \(\psi_{k-1}\) is the dynamical model, and \(h_k\) is a
potentially nonlinear observation operator. The model-error law is left
general, whereas the additive observation error satisfies
\begin{equation}
\eta_k\sim\mathcal{N}(0,R_k),
\qquad R_k\succ 0.
\label{eq:obs_noise}
\end{equation}
The observation noises are assumed independent across assimilation cycles and
independent of the state and model-noise processes.
Let \(p_k(x\mid x'):=p(X_k=x\mid X_{k-1}=x')\) denote the Markov transition
density, and let \(y_{1:k}=(y_1,\ldots,y_k)\) denote the realized observations.

Given \(y_{1:k-1}\), the forecast density is
\begin{equation}
p_k^f(x)
:=
p(x\mid y_{1:k-1})
=
\int p_k(x\mid x')p_{k-1}^a(x')\,dx'.
\label{eq:forecast_distribution}
\end{equation}
Given \(y_k\), Bayes' rule gives the analysis density
\begin{equation}
p_k^a(x)
:=
p(x\mid y_{1:k})
=
\frac{\ell_k(x;y_k)p_k^f(x)}{\mathcal Z_k(y_k)},
\qquad
\mathcal Z_k(y_k)
=
\int \ell_k(x';y_k)p_k^f(x')\,dx',
\label{eq:bayes_analysis}
\end{equation}
where \(\ell_k(x;y_k):=p(y_k\mid x)\) is the observation likelihood.

\method{} does not assume direct access to \(p_k^f\). Instead, the current
forecast ensemble
\begin{equation}
\mathcal E_k^f=\{x_k^{f,(i)}\}_{i=1}^{M}
\label{eq:forecast_ensemble}
\end{equation}
defines a finite-ensemble surrogate \(\widetilde p_k^f\), and the corresponding
surrogate analysis is
\begin{equation}
\widetilde p_k^a(x)
\propto
\ell_k(x;y_k)\widetilde p_k^f(x).
\label{eq:surrogate_analysis}
\end{equation}
Section~\ref{sec:method} constructs a posterior flow for this surrogate target,
while Section~\ref{sec:theory} separates finite-flow error from
forecast-surrogate error.

\section{FREESIA Posterior Flow}
\label{sec:method}
\begingroup
\setlength{\abovedisplayskip}{2.5pt plus 1pt minus 1pt}
\setlength{\belowdisplayskip}{2.5pt plus 1pt minus 1pt}
\setlength{\abovedisplayshortskip}{0.5pt plus 1pt minus 0.5pt}
\setlength{\belowdisplayshortskip}{0.5pt plus 1pt minus 0.5pt}
\setlength{\intextsep}{4pt plus 2pt minus 2pt}
\setlength{\abovecaptionskip}{1pt}
\setlength{\belowcaptionskip}{0pt}

This section derives the FREESIA posterior flow for a single assimilation cycle
from the current forecast ensemble. We use $k$ for the assimilation-cycle
index, $t$ for probability-flow time, $j$ for GMM components, and $s$ for
endpoint importance samples. Within one cycle, we suppress $k$, write the
forecast ensemble as $\mathcal E^f=\{x_j^f\}_{j=1}^M$, and set $h:=h_k$,
$R:=R_k$, and $y:=y_k$.

The posterior-flow path targets the surrogate analysis distribution
$\widetilde p^a$ defined in Section~\ref{sec:problem}, with $Z_0,Z_t,Z_1$
denoting auxiliary flow variables rather than the physical filtering state
$X_k$. The observation is associated with the endpoint $Z_1$ through
$Y=h(Z_1)+\eta$, where $\eta\sim\mathcal N(0,R)$ is independent of $Z_0$ and
$Z_1$. FREESIA constructs a localized Gaussian-mixture surrogate of the forecast
distribution from the current ensemble. Its shared covariance structure yields
an analytically tractable conditional endpoint distribution that retains the
forecast cross-covariances used for sparse information transfer.

\begin{figure}[H]
    \centering
    \includegraphics[width=0.61\linewidth]{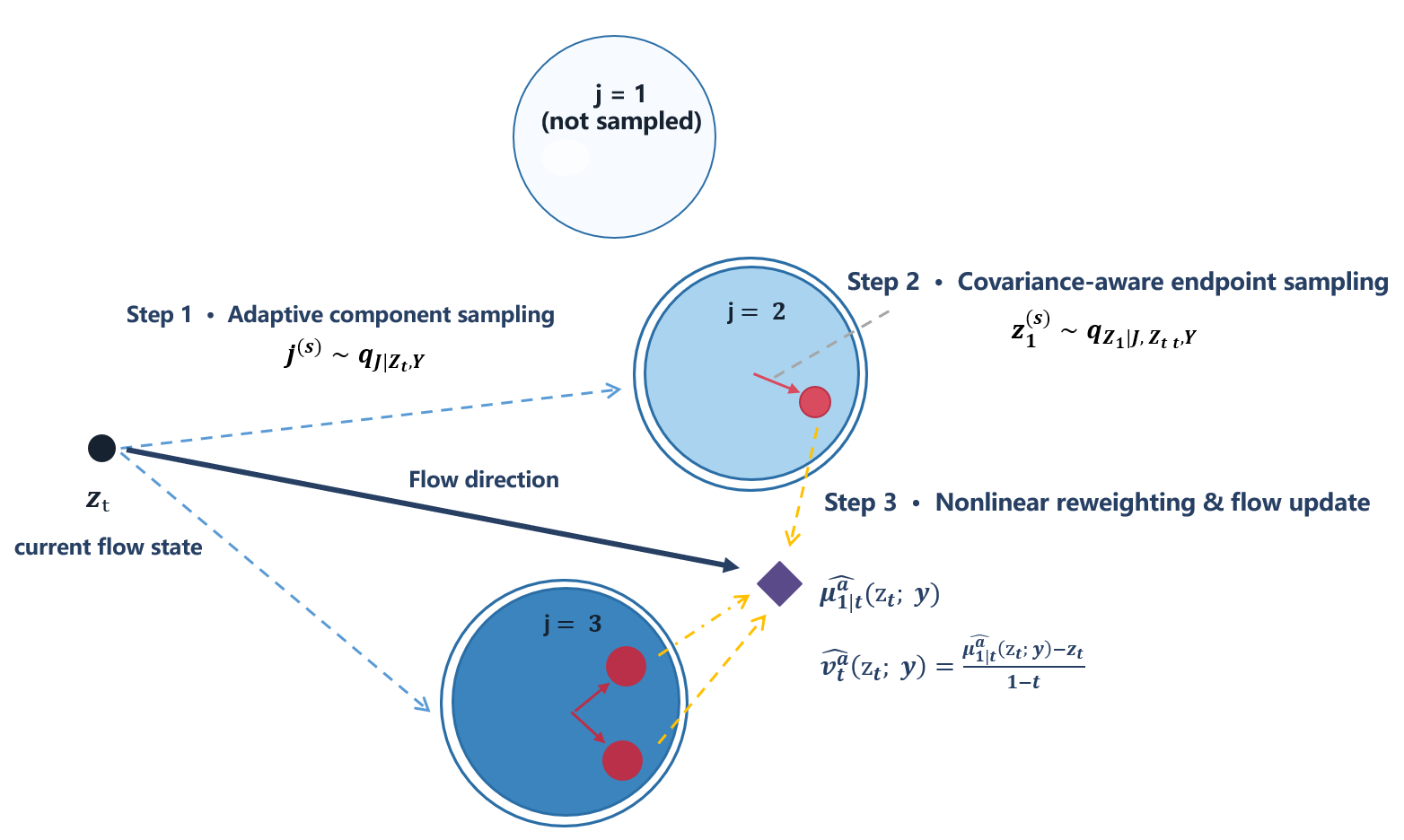}
    \caption{Overview of the FREESIA posterior-flow update. Given $z_t$, Step~1 samples an observation-adaptive GMM component, Step~2 draws a covariance-aware conditional endpoint, and Step~3 applies nonlinear likelihood-ratio correction to estimate the endpoint mean and advance the flow.}
    \label{fig:freesia-posterior-velocity}
\end{figure}

\subsection{Posterior-Flow Construction with a Gaussian-Mixture Surrogate}
\label{sec:gmm-bridge}

This subsection derives the conditional endpoint distribution along the
straight probability path
\begin{equation}
Z_t=(1-t)Z_0+tZ_1,
\qquad t\in[0,1],
\label{eq:straight_path}
\end{equation}
where \(Z_1\sim\widetilde p^f\) follows the forecast surrogate and \(Z_0\) is an
independent source variable. Conditioning this path on the current observation
defines the surrogate posterior path
\begin{equation}
\widetilde p_t^a:=\mathcal L(Z_t\mid Y=y),
\qquad t\in[0,1],
\qquad
\widetilde p_1^a=\widetilde p^a.
\label{eq:surrogate-posterior-path}
\end{equation}
Since $Z_0$ is independent of $Y$, the path starts from the source distribution
$\mathcal L(Z_0)$ and terminates at the surrogate posterior $\widetilde p^a$.
The velocity is given by conditional endpoint regression.

\begin{theorem}
\label{thm:posterior-flow}
For a fixed observation \(Y=y\) and \(t\in[0,1)\),
\begin{equation}
v_t^a(z_t;y)
=
\frac{
\mathbb{E}[Z_1\mid Z_t=z_t,Y=y]-z_t
}{
1-t
}.
\label{eq:posterior_velocity}
\end{equation}
\end{theorem}

The proof is given in Appendix~\ref{app:proof-posterior-flow}. We write $v_t^a$
and $\mu_{1\mid t}^a$ for the velocity and conditional endpoint mean associated
with the surrogate posterior path $\widetilde p_t^a$. The theorem therefore
reduces posterior-flow construction to estimating the conditional endpoint mean
\[
\mu_{1\mid t}^a(z_t;y)
:=\mathbb{E}[Z_1\mid Z_t=z_t,Y=y].
\]

To evaluate this mean, we specify the ensemble-GMM surrogate. Given
$\mathcal E^f=\{x_j^f\}_{j=1}^M$, assign equal mixture weights
$\pi_j=1/M$, $j=1,\ldots,M$. Hence
\begin{equation}
\bar x^f=\frac1M\sum_{j=1}^M x_j^f,
\qquad
\mu_j=(1-\gamma)\bar x^f+\gamma x_j^f,
\qquad 0\le\gamma<1.
\label{eq:gmm-component-means}
\end{equation}
Let \(\Sigma^f\succ0\) be the localized, inflated covariance described in
Appendix~\ref{app:covariance-implementation}. We specify \(Z_0\) and \(Z_1\) as
follows, with the mixture law of \(Z_1\) represented hierarchically by \(J\) and
\(Z_1\mid J\):
\begin{equation}
\label{eq:forecast-gmm-decomposition}
\begin{aligned}
Z_0&\sim\mathcal N(0,\Sigma^f),\qquad
Z_1\sim\frac1M\sum_{j=1}^M\mathcal N(\mu_j,(1-\gamma^2)\Sigma^f),
\qquad Z_0\perp(Z_1,J),\\
J&\sim\operatorname{Cat}\!\left(\frac1M,\ldots,\frac1M\right),\qquad
Z_1\mid J=j\sim\mathcal N(\mu_j,(1-\gamma^2)\Sigma^f).
\end{aligned}
\end{equation}
Together with the straight path, these Gaussian endpoint distributions yield
a closed-form conditional endpoint distribution. It decomposes as follows.
Here $r_{J\mid Z_t}$ denotes the conditional component probabilities, and
$\phi_{1\mid t,j}$ denotes the component conditional densities:
\begin{equation}
\label{eq:conditional-gmm-decomposition}
\begin{aligned}
p_{Z_1\mid Z_t}(z_1\mid z_t)
&=\sum_{j=1}^M r_{J\mid Z_t}(j\mid z_t)\,
\phi_{1\mid t,j}(z_1\mid z_t),\\
r_{J\mid Z_t}(j\mid z_t)
&=
\frac{
\pi_j\mathcal N\!\left(
z_t;t\mu_j,[(1-t)^2+(1-\gamma^2)t^2]\Sigma^f
\right)
}{
\sum_\ell\pi_\ell\mathcal N\!\left(
z_t;t\mu_\ell,[(1-t)^2+(1-\gamma^2)t^2]\Sigma^f
\right)
},\\
\phi_{1\mid t,j}(z_1\mid z_t)
&:=\mathcal N(z_1;\mu_{1\mid t,j}(z_t),\Sigma_{1\mid t}),\\
\mu_{1\mid t,j}(z_t)
&=
\mu_j+
\frac{t(1-\gamma^2)}{(1-t)^2+(1-\gamma^2)t^2}
(z_t-t\mu_j),\\
\Sigma_{1\mid t}
&=
\frac{(1-\gamma^2)(1-t)^2}{(1-t)^2+(1-\gamma^2)t^2}
\Sigma^f.
\end{aligned}
\end{equation}
The conditional means $\mu_{1\mid t,j}(z_t)$ determine the component-wise
endpoint locations along the path, while $\Sigma_{1\mid t}$ retains the
cross-coordinate covariance structure needed to propagate sparse observational
information in the subsequent conditioning step.

\subsection{Nonlinear Conditioning and Finite-Sample Approximation}
\label{sec:nonlinear-proposal}

Direct Monte Carlo estimation of $\mu^a_{1\mid t}(z_t;y)$ using samples from
the observation-independent bridge can be inefficient when the observation
likelihood is highly concentrated relative to the bridge distribution.
\method{} instead constructs an observation-adaptive locally linearized
proposal distribution and corrects it using the exact nonlinear likelihood
ratio. The resulting computation consists of observation-adaptive component
sampling, covariance-aware endpoint sampling, and nonlinear likelihood-ratio
correction to estimate the conditional endpoint mean and posterior-flow velocity.

For each component and flow state $z_t$, locally linearize $h$ at
$\mu_{1\mid t,j}(z_t)$, using
$H_{t,j}(z_t):=Dh(\mu_{1\mid t,j}(z_t))$ at differentiable points and a
bounded local linearization selection at nondifferentiable points;
Appendix~\ref{app:is-control} covers both cases. The exact and component-wise
linearized likelihoods are
\begin{equation}
\label{eq:nonlinear-linearized-likelihood}
\begin{aligned}
p(y\mid z_1)&=\mathcal N(y;h(z_1),R),\\
p_{\mathrm{lin}}(y\mid z_1,j,z_t)
&=\mathcal N\!\left(
y;
h(\mu_{1\mid t,j}(z_t))+
H_{t,j}(z_t)(z_1-\mu_{1\mid t,j}(z_t)),
R
\right).
\end{aligned}
\end{equation}
The nonlinear likelihood defines the target, whereas the local linearization is
used only to construct the proposal distribution.

Conditioned on the current flow state \(z_t\), the nonlinear augmented target
and its locally linearized proposal are
\begin{equation}
\label{eq:augmented-target-proposal}
\begin{aligned}
p_{J,Z_1\mid Z_t,Y}(j,z_1\mid z_t,y)
&\propto
r_{J\mid Z_t}(j\mid z_t)\,
\phi_{1\mid t,j}(z_1\mid z_t)\,
p(y\mid z_1),\\
q_{J,Z_1\mid Z_t,Y}(j,z_1\mid z_t,y)
&=
q_{J\mid Z_t,Y}(j\mid z_t,y)\,
q_{Z_1\mid J,Z_t,Y}(z_1\mid j,z_t,y)\\
&\propto
r_{J\mid Z_t}(j\mid z_t)\,
\phi_{1\mid t,j}(z_1\mid z_t)\,
p_{\mathrm{lin}}(y\mid z_1,j,z_t).
\end{aligned}
\end{equation}
Gaussian marginalization and conditioning give
\begin{equation}
\label{eq:proposal-parameters}
\begin{aligned}
S_{t,j}(z_t)
&=H_{t,j}(z_t)\Sigma_{1\mid t}H_{t,j}(z_t)^\top+R,\\
G_{t,j}(z_t)
&=\Sigma_{1\mid t}H_{t,j}(z_t)^\top S_{t,j}(z_t)^{-1},\\
\mu^q_{1\mid t,j}(z_t;y)
&=\mu_{1\mid t,j}(z_t)
+G_{t,j}(z_t)[y-h(\mu_{1\mid t,j}(z_t))],\\
\Sigma^q_{1\mid t,j}(z_t)
&=\Sigma_{1\mid t}
-G_{t,j}(z_t)H_{t,j}(z_t)\Sigma_{1\mid t}.
\end{aligned}
\end{equation}
\begin{equation}
\label{eq:proposal-components}
\begin{aligned}
q_{J\mid Z_t,Y}(j\mid z_t,y)
&\propto
r_{J\mid Z_t}(j\mid z_t)\,
\mathcal N\!\left(y;h(\mu_{1\mid t,j}(z_t)),S_{t,j}(z_t)\right),\\
q_{Z_1\mid J,Z_t,Y}(z_1\mid j,z_t,y)
&=
\mathcal N\!\left(
z_1;\mu^q_{1\mid t,j}(z_t;y),\Sigma^q_{1\mid t,j}(z_t)
\right).
\end{aligned}
\end{equation}
For brevity, we suppress dependence on $z_t$ and $y$ where no ambiguity arises.
The categorical proposal $q_{J\mid Z_t,Y}$ favors mixture components compatible
with the observation, while $G_{t,j}$ propagates observational information
through the conditional covariance $\Sigma_{1\mid t}$. For compactness, write
$r_j:=r_{J\mid Z_t}(j\mid z_t)$,
$\phi_j(z_1):=\phi_{1\mid t,j}(z_1\mid z_t)$,
$q(j,z_1):=q_{J,Z_1\mid Z_t,Y}(j,z_1\mid z_t,y)$, and
$g(j,z_1):=p(y\mid z_1)/p_{\mathrm{lin}}(y\mid z_1,j,z_t)$.

\begin{theorem}
\label{thm:nonlinear-regression}
For the nonlinear target and proposal defined in
Eq.~\ref{eq:augmented-target-proposal},
\begin{equation}
\begin{aligned}
\mu^a_{1\mid t}(z_t;y)
&=
\frac{\sum_j r_j\int z_1\phi_j(z_1)p(y\mid z_1)\,dz_1}
     {\sum_j r_j\int \phi_j(z_1)p(y\mid z_1)\,dz_1} \\
&=
\frac{\sum_j\int z_1 q(j,z_1)g(j,z_1)\,dz_1}
     {\sum_j\int q(j,z_1)g(j,z_1)\,dz_1}.
\end{aligned}
\label{eq:corrected-endpoint-mean}
\end{equation}
\end{theorem}

At the population level, likelihood-ratio reweighting exactly recovers the
nonlinear surrogate target; local linearization affects only the proposal. The
proof is in Appendix~\ref{app:proof-nonlinear-correction}.

\begin{corollary}
\label{cor:self-normalized-regression}
For $s=1,\ldots,N_{\mathrm{IS}}$, let
$J^{(s)}\sim q_{J\mid Z_t,Y}(\cdot\mid z_t,y)$ and
$Z_1^{(s)}\sim q_{Z_1\mid J,Z_t,Y}(\cdot\mid J^{(s)},z_t,y)$.
Define
\begin{equation}
\widetilde w^{(s)}=g(J^{(s)},Z_1^{(s)}),
\qquad
w^{(s)}=\frac{\widetilde w^{(s)}}{\sum_{r=1}^{N_{\mathrm{IS}}}\widetilde w^{(r)}}.
\label{eq:snis-weights}
\end{equation}
Then
\begin{equation}
\hat\mu^a_{1\mid t,N_{\mathrm{IS}}}(z_t;y)
=\sum_{s=1}^{N_{\mathrm{IS}}}w^{(s)}Z_1^{(s)},
\qquad
\hat v^a_{t,N_{\mathrm{IS}}}(z_t;y)
=\frac{\hat\mu^a_{1\mid t,N_{\mathrm{IS}}}(z_t;y)-z_t}{1-t}.
\label{eq:finite-endpoint-regression}
\end{equation}
\end{corollary}

Appendix~\ref{app:is-control} proves pointwise consistency under first-moment
conditions and, under the stronger accumulated self-normalized importance
sampling (SNIS) regularity condition, controls the flow-integrated finite-sample contribution at rate
$O(N_{\mathrm{IS}}^{-1/2})$. Starting
from source particles $Z_0^{(i)}\sim\mathcal N(0,\Sigma^f)$, we integrate the
posterior-flow ODE using the velocity estimator in
Eq.~\ref{eq:finite-endpoint-regression} and take the terminal particles as the
analysis ensemble. The complete finite-ensemble algorithm,
structured-covariance implementation, and leading-order computational and
memory complexity are detailed in Appendix~\ref{app:covariance-implementation}.
\FloatBarrier
\endgroup

\section{THEORETICAL ANALYSIS}
\label{sec:theory}

This section quantifies the one-step error of the finite \method{} update
relative to the Bayesian posterior $p^a$. We separate the error arising from
finite posterior-flow computation from the discrepancy introduced by the
ensemble-based forecast surrogate.

For one assimilation cycle, the underlying and surrogate posterior densities are
\[
p^a(x)\propto p(y\mid x)p^f(x),
\qquad
\widetilde p^a(x)\propto p(y\mid x)\widetilde p^f(x).
\]
Define the surrogate-posterior discrepancy by
\begin{equation}
\varepsilon_{\mathrm{sur}}(y)
:=W_1(\widetilde p^a,p^a).
\label{eq:surrogate-error}
\end{equation}
Let $\widehat p^a_{N_{\mathrm{IS}},T}$ denote the law produced by the finite
\method{} update using $N_{\mathrm{IS}}$ endpoint proposals per velocity
evaluation and $T$ Euler flow steps.

The total one-step error decomposes as
\begin{equation}
W_1\!\left(\widehat p^a_{N_{\mathrm{IS}},T},p^a\right)
\le
W_1\!\left(\widehat p^a_{N_{\mathrm{IS}},T},\widetilde p^a\right)
+\varepsilon_{\mathrm{sur}}(y).
\label{eq:error-split}
\end{equation}

Although Eq.~\ref{eq:posterior_velocity} contains the factor $1/(1-t)$, the
conditional Gaussian bridge removes the apparent endpoint singularity.
The finite-sampling error can be controlled directly for bounded nonlinear
and affine observation maps. Appendix~\ref{app:is-control} gives the proofs
and retains an accumulated-SNIS formulation for more general unbounded
nonlinear observations.
\begin{theorem}
\label{thm:total-error}
Suppose the ideal posterior velocity $v_t^a$ is uniformly $L_v$-Lipschitz in
the state for $t\in[0,1]$, and its exact flow has a uniform one-step Euler
truncation error of order $O((\Delta t)^2)$. For bounded nonlinear
observations, assume
\[
B_y:=\sup_x\|y-h(x)\|_{R^{-1}}<\infty,
\]
and
\[
\kappa:=(1-\gamma^2)
\sup_x\lambda_{\max}\!\left(
(\Sigma^f)^{1/2}H(x)^\top R^{-1}H(x)(\Sigma^f)^{1/2}
\right)<\infty.
\]
Here $H(x)$ denotes the local observation linearization used in
Section~3.2, and $\lambda_{\max}(\cdot)$ denotes the
largest eigenvalue. For every $0<\delta\le1$ satisfying $\delta\kappa<1$,
\begin{equation}
\label{eq:total-error-bound}
W_1\!\left(\widehat p^a_{N_{\mathrm{IS}},T},p^a\right)
\le
C_\delta N_{\mathrm{IS}}^{-\delta/(1+\delta)}
+C_2T^{-1}
+\varepsilon_{\mathrm{sur}}(y).
\end{equation}
If instead $h(x)=Ax+b$ is affine, its local linearization is exact and the
likelihood-ratio correction satisfies $g\equiv1$. In this case,
\begin{equation}
\label{eq:affine-total-error-bound}
W_1\!\left(\widehat p^a_{N_{\mathrm{IS}},T},p^a\right)
\le
C_1N_{\mathrm{IS}}^{-1/2}
+C_2T^{-1}
+\varepsilon_{\mathrm{sur}}(y),
\end{equation}
without requiring bounded observation residuals or an importance-weight
overlap condition. The constants are independent of $N_{\mathrm{IS}}$ and $T$.
\end{theorem}

\noindent\textbf{Corollary.}
In the bounded nonlinear regime, if $\kappa<1$, taking $\delta=1$ in
Theorem~\ref{thm:total-error} yields the standard $N_{\mathrm{IS}}^{-1/2}$
finite-sampling rate.

The first term is finite endpoint-sampling error, the second is first-order
flow-discretization error, and $\varepsilon_{\mathrm{sur}}(y)$ is the
forecast-surrogate discrepancy defined in Eq.~\ref{eq:surrogate-error}.
Thus the full-evidence transport converges to $\widetilde p^a$ as
$N_{\mathrm{IS}},T\to\infty$ under either observation regime above. For more
general unbounded nonlinear observations, Appendix~\ref{app:is-control} gives
the accumulated-SNIS alternative. Convergence to the underlying posterior
additionally requires $\varepsilon_{\mathrm{sur}}(y)\to0$. The result remains
one-step and law-level. The diagonal component-evidence approximation used in
the high-dimensional experiments is analyzed separately in
Appendix~\ref{app:diagonal-evidence-error}.

\begingroup
\setlength{\textfloatsep}{3pt plus 1pt minus 1pt}
\setlength{\floatsep}{3pt plus 1pt minus 1pt}
\setlength{\intextsep}{3pt plus 1pt minus 1pt}
\setlength{\abovecaptionskip}{2pt}
\setlength{\belowcaptionskip}{1pt}
\setlength{\parskip}{0pt}
\setlength{\abovedisplayskip}{4pt plus 1pt minus 1pt}
\setlength{\belowdisplayskip}{4pt plus 1pt minus 1pt}

\section{Experiments}

We evaluate FREESIA along three increasingly difficult dimensions: posterior
multimodality, sparse-observation information transfer, and high-dimensional
nonlinear assimilation. Double-Well provides a grid-resolved posterior,
Lorenz--96 separates many-to-one conditioning from sparse transfer at 1000
dimensions, and Kolmogorov flow combines both challenges at 32768 dimensions.
Metrics average independent trajectories and analysis times.
Appendix~\ref{app:highdim-scaling} summarizes computational complexity and
empirical analysis-time scaling for training-free generative filters on fully
observed Lorenz--96 up to $d=10^6$, including a diagonal specialization of
FREESIA for extremely high-dimensional settings.
Probabilistic accuracy is evaluated using the continuous ranked probability
score (CRPS) \citep{brocker2012crps}; the Double-Well posterior comparison
additionally reports the Jensen--Shannon divergence \citep{lin1991js}.
IEnSF is not included in the numerical comparison because, although we
independently reproduced the method, we could not verify that the resulting
performance was sufficiently representative of the published method for a fair
comparison.

\subsection{Double-Well: Resolving a Multimodal Filtering Posterior}

This experiment isolates posterior multimodality in a setting where the full
filtering posterior is numerically available, allowing direct evaluation of
posterior shape rather than only pointwise state error.

\paragraph{Setup.}
Following \citet{ding2026ssls}, we use the double-well dynamics and many-to-one
observation \(Y_k=|X_k-0.4|+\epsilon_k\). A 10,000-point grid Bayes filter
provides the posterior reference, and metrics average \(N_{\mathrm{run}}=10\)
runs; complete dynamical and numerical settings are given in Appendix~B.1.

\paragraph{Results.}
FREESIA is closest to the grid Bayes posterior across all reported metrics and
tracks both posterior modes and their relative masses
(Table~\ref{tab:double-well}; Appendix Fig.~\ref{fig:double_well_posterior}). EnKF
and EnKF-MDA Gaussianize the analysis, while EnFF-F2P can concentrate on only
part of the posterior support, so the gain is distributional rather than only a
reduction in posterior-mean error.

\begin{table}[!htbp]
\centering
\small
\renewcommand{\arraystretch}{0.82}
\setlength{\aboverulesep}{0.20ex}
\setlength{\belowrulesep}{0.20ex}
\setlength{\tabcolsep}{4.5pt}
\caption{Double-well posterior approximation errors averaged over experiments and analysis times. Lower is better; best and second-best results are bold and underlined.}
\label{tab:double-well}
\begin{tabular}{lccccc}
\toprule
Method & JS $\downarrow$ & $W_1$ $\downarrow$ & Mode err. $\downarrow$ & RMSE $\downarrow$ & CRPS $\downarrow$ \\
\midrule
\textbf{\method{}} (ours) & \textbf{0.031} & \textbf{0.059} & \textbf{0.035} & \textbf{0.144} & \textbf{0.099} \\
EnKF-MDA & 0.185 & 0.320 & 0.259 & 0.433 & 0.374 \\
EnKF & \underline{0.184} & 0.233 & 0.158 & 0.325 & 0.257 \\
EnSF & 0.315 & 0.666 & 0.359 & 0.699 & 0.379 \\
EnFF-OT & 0.550 & 0.202 & 0.130 & 0.263 & 0.254 \\
EnFF-F2P & 0.358 & \underline{0.152} & \underline{0.085} & \underline{0.232} & \underline{0.170} \\
\bottomrule
\end{tabular}
\end{table}

\paragraph{Sensitivity.}
Appendix~B.1 reports sensitivity to $\gamma$, $T$, and $N_{\mathrm{IS}}$:
component separation is important in this multimodal regime, while roughly 20
flow steps capture most of the improvement.

\FloatBarrier
\subsection{Lorenz--96: Multimodality and Sparse Information Propagation}

Lorenz--96 separates the two main mechanisms in a 1000-dimensional chaotic
system: dense many-to-one observations test nonlinear posterior ambiguity,
while sparse monotone observations isolate covariance-mediated information
transfer.

\paragraph{Setup.}
We use \(d=1000\), 20 particles, and 101 analyses under three regimes: dense
\(|\arctan(x_i)|\), stride-2 \(\arctan(x_{2i})\), and stride-4
\(\arctan(x_{4i})\). Results average \(N_{\mathrm{run}}=10\) runs; time
stepping, noise levels, baselines, localization, and method-specific settings
are given in Appendix~B.2.
The stride-4 case uses $\sigma=0.01$ as a high-SNR sparse-observation stress
test; Appendix~B.2 also reports the $\sigma=0.05$ setting.

\paragraph{Dense many-to-one observations.}
With every coordinate observed through $|\arctan|$, FREESIA achieves RMSE
$0.490$ and CRPS $0.174$, substantially below all competitors in
Table~\ref{tab:l96-results},
extending the multimodal-conditioning advantage to a 1000-dimensional chaotic
system.

\begin{table}[!htbp]
\centering
\small
\renewcommand{\arraystretch}{1.00}
\setlength{\aboverulesep}{0.30ex}
\setlength{\belowrulesep}{0.30ex}
\setlength{\tabcolsep}{1.7pt}
\caption{Lorenz--96 filtering under three nonlinear observation regimes. RMSE and CRPS are averaged across repeated experiments and all 101 assimilation cycles. Lower is better; best and second-best values are bold and underlined.}
\label{tab:l96-results}
\begin{tabular}{lcccccccccc}
\toprule
& \multicolumn{2}{c}{Dense $|\arctan|$} & \multicolumn{4}{c}{Stride-2 $\arctan$} & \multicolumn{4}{c}{Stride-4 $\arctan$} \\
\cmidrule(lr){2-3}\cmidrule(lr){4-7}\cmidrule(lr){8-11}
Method & RMSE $\downarrow$ & CRPS $\downarrow$ & RMSE $\downarrow$ & Obs. $\downarrow$ & Unobs. $\downarrow$ & CRPS $\downarrow$ & RMSE $\downarrow$ & Obs. $\downarrow$ & Unobs. $\downarrow$ & CRPS $\downarrow$ \\
\midrule
\textbf{\method{}} (ours) & \textbf{0.490} & \textbf{0.174} & \textbf{0.615} & \textbf{0.417} & \textbf{0.757} & \underline{0.241} & \textbf{1.575} & \underline{0.373} & \textbf{1.805} & \textbf{0.584} \\
EnKF-MDA & \underline{1.384} & \underline{0.586} & \underline{0.987} & \underline{0.673} & \underline{1.211} & \textbf{0.232} & \underline{2.056} & \textbf{0.366} & \underline{2.358} & \underline{0.955} \\
EnKF & 3.053 & 1.712 & 1.380 & 1.120 & 1.573 & 0.422 & 2.684 & 1.689 & 2.928 & 1.356 \\
LETKF & 6.161 & 3.921 & 1.533 & 1.243 & 1.753 & 0.494 & 3.185 & 1.853 & 3.513 & 1.648 \\
EnSF & 3.901 & 2.340 & 2.027 & 1.531 & 2.422 & 0.976 & 2.472 & 1.143 & 2.777 & 1.242 \\
EnFF-OT & 4.341 & 2.959 & 2.817 & 2.303 & 3.247 & 1.825 & 3.749 & 2.637 & 4.051 & 2.636 \\
EnFF-F2P & 1.933 & 0.721 & 1.486 & 1.045 & 1.821 & 0.609 & 3.960 & 4.412 & 3.796 & 2.243 \\
\bottomrule
\end{tabular}
\end{table}

\paragraph{Sparse monotone observations.}
FREESIA has the lowest total RMSE in both sparse regimes. At stride 4, EnKF-MDA
is marginally better on observed coordinates ($0.366$ versus $0.373$), but
FREESIA is substantially better on the much larger unobserved subset ($1.805$
versus $2.358$). Because both sparse experiments use $\gamma=0$, this gain
isolates covariance-aware information propagation rather than mixture
separation. Stride-4 ablations raise the time-averaged RMSE from $1.575$ to
$2.352$ with diagonal covariance and $2.969$ without likelihood-ratio correction
(Figure~\ref{fig:l96-stride4-ablation}); Appendix~B.2 gives details and the
$\sigma=0.05$ robustness result.

\FloatBarrier
\subsection{Kolmogorov Flow: Sparse and Non-Injective PDE Assimilation}

This benchmark combines the two challenges studied separately in Lorenz--96 at
PDE scale: extreme observation sparsity and non-injective nonlinear
measurements.

\paragraph{Setup.}
We use a periodically forced incompressible Navier--Stokes benchmark on a
periodic \(128\times128\) grid, giving \(d=32768\) from the two velocity
components, with 20 particles. We observe either 1\% of velocity entries through
\(\arctan(10x)\) or 2\% through \(|\arctan(10x)|\), averaging
\(N_{\mathrm{run}}=10\) trajectories; complete settings are given in
Appendix~B.3.
The block-structured forecast-surrogate covariance is detailed in
Appendix~\ref{app:covariance-implementation}, and the reported comparison uses
\(N_{\mathrm{IS}}=5\).

\begin{table}[!htbp]
\centering
\small
\renewcommand{\arraystretch}{1.00}
\setlength{\aboverulesep}{0.30ex}
\setlength{\belowrulesep}{0.30ex}
\setlength{\tabcolsep}{4pt}
\caption{Kolmogorov-flow filtering under sparse nonlinear observations. Metrics average ten trajectories and all 51 analyses; lower is better, with best and second-best values bold and underlined.}
\label{tab:2dkf-results}
\begin{tabular}{lcccccccc}
\toprule
& \multicolumn{4}{c}{1\% random $\arctan$} & \multicolumn{4}{c}{2\% random $|\arctan|$} \\
\cmidrule(lr){2-5}\cmidrule(lr){6-9}
Method & RMSE $\downarrow$ & Obs. $\downarrow$ & Unobs. $\downarrow$ & CRPS $\downarrow$ & RMSE $\downarrow$ & Obs. $\downarrow$ & Unobs. $\downarrow$ & CRPS $\downarrow$ \\
\midrule
\textbf{\method{}} (ours) & \textbf{0.095} & \textbf{0.076} & \textbf{0.095} & \textbf{0.052} & \textbf{0.422} & \textbf{0.410} & \textbf{0.422} & \textbf{0.209} \\
EnKF-MDA & \underline{0.148} & \underline{0.115} & \underline{0.148} & \underline{0.081} & 1.121 & 1.115 & 1.121 & 0.797 \\
EnKF & 0.377 & 0.372 & 0.377 & 0.200 & 0.987 & 0.984 & 0.987 & 0.681 \\
LETKF & 0.218 & 0.177 & 0.218 & 0.117 & \underline{0.970} & 1.088 & \underline{0.967} & 0.572 \\
EnSF & 0.942 & 0.568 & 0.945 & 0.547 & 0.974 & \underline{0.967} & 0.975 & \underline{0.569} \\
EnFF-OT & 1.288 & 1.148 & 1.290 & 1.024 & 1.342 & 1.342 & 1.342 & 1.069 \\
EnFF-F2P & 0.974 & 0.776 & 0.976 & 0.567 & 0.980 & 0.977 & 0.980 & 0.570 \\
\bottomrule
\end{tabular}
\end{table}

\paragraph{Sparse monotone observations.}
With only $1\%$ of velocity entries observed through the monotone
$\arctan(10x)$ operator, FREESIA is the only training-free generative method in
the comparison that reaches a strongly reduced-error regime. Its time-averaged
RMSE is $0.095$ with an unobserved-state RMSE of $0.095$, whereas EnSF, EnFF-OT,
and EnFF-F2P remain near order-one error with RMSEs $0.942$, $1.288$, and
$0.974$, respectively. Classical covariance-based filters can also reduce the
error, but FREESIA remains best overall, improving over the strongest baseline
EnKF-MDA from RMSE $0.148$ to $0.095$ and achieving the lowest CRPS ($0.052$).
This result extends covariance-mediated sparse information transfer to a
$32768$-dimensional nonlinear PDE while retaining non-Gaussian posterior
transport.

\paragraph{Sparse non-injective observations.}
In the $2\%$ $|\arctan|$ regime, FREESIA achieves RMSE $0.422$ and CRPS
$0.209$, compared with the lowest evaluated baseline RMSE of $0.970$ and the
lowest evaluated baseline CRPS of $0.569$. Figure~\ref{fig:kolmogorov_reconstruction} further shows recovery of
the dominant spatial structures and physically relevant velocity directions
under sparse sign-ambiguous observations. The time-resolved comparison is
reported in Appendix~B.3.

\begin{figure}[H]
\centering
\includegraphics[width=0.425\linewidth]{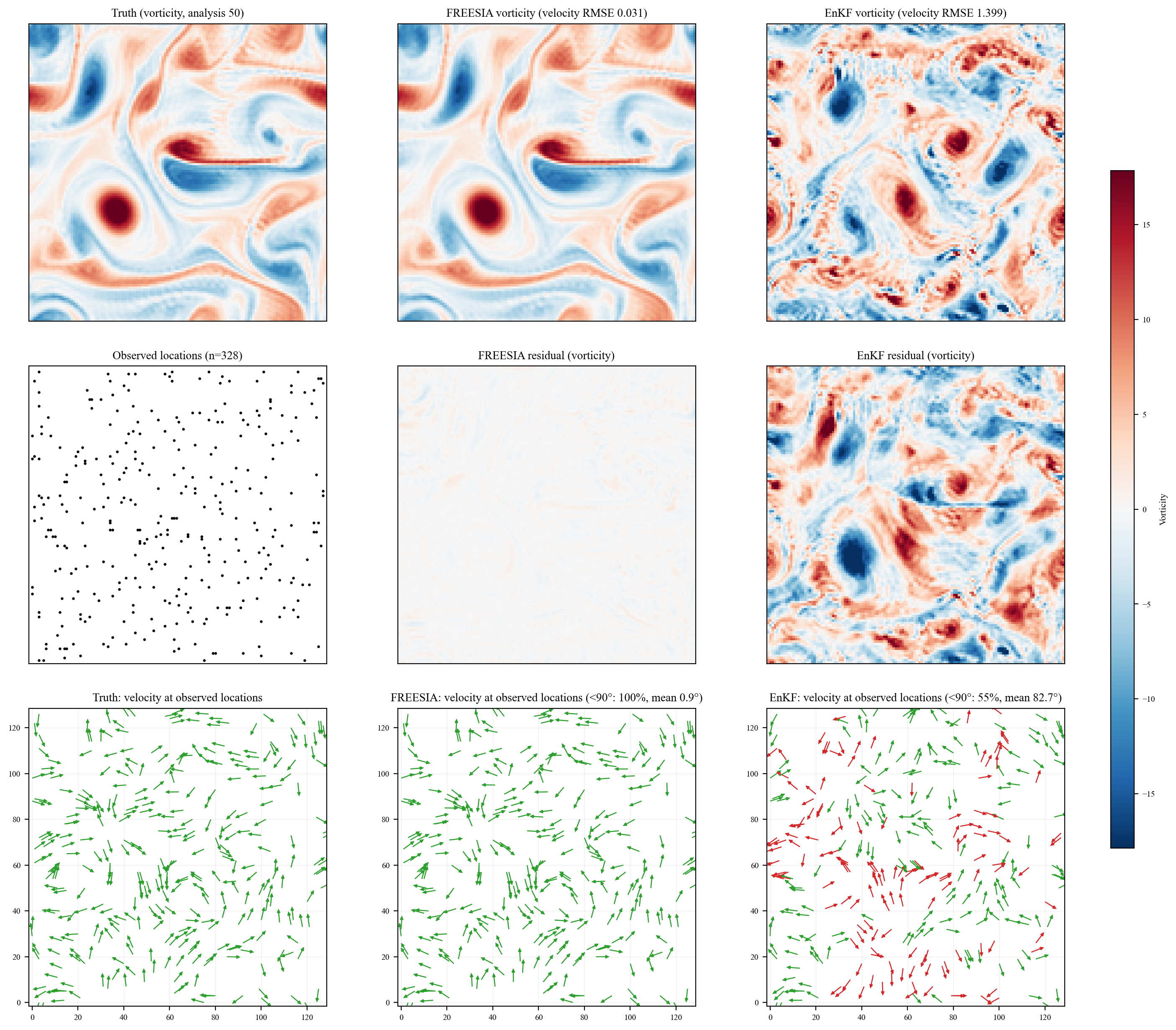}
\caption{Final-time reconstruction for the 2\% absolute-arctangent Kolmogorov-flow experiment. Top: truth, FREESIA, and EnKF vorticity, with analysis-mean velocity RMSE $0.031$ vs.\ $1.399$. Middle: 328 observed spatial locations, corresponding to 656 observed velocity components (2\% of the state), and vorticity residuals. Bottom: observed-point velocity directions; FREESIA/EnKF achieve $100\%/55\%$ within $90^\circ$, with mean angular errors $0.9^\circ/82.7^\circ$.}
\label{fig:kolmogorov_reconstruction}
\end{figure}

\endgroup

\section{Conclusion}
\begingroup
\setlength{\parskip}{0pt}
We propose FREESIA, a training-free, covariance-aware posterior transport
framework for high-dimensional nonlinear data assimilation. By embedding
forecast cross-covariances into flow-based transport, FREESIA propagates sparse
observational information to unobserved states while retaining non-Gaussian
posterior structure. Observation-adaptive proposals and posterior correction
preserve the nonlinear surrogate target. We establish surrogate-relative
asymptotic exactness and derive a one-step Wasserstein error bound under the
stated regularity conditions, separating finite-sampling,
flow-discretization, and forecast-surrogate errors.

Experiments on Double-Well, Lorenz--96, and Kolmogorov flow show that these
mechanisms translate into more accurate posterior approximation, ensemble
sampling, and state estimation under multimodal, sparse, and non-injective
observations. FREESIA more accurately recovers posterior modes and their
probability mass, improves inference of unobserved states through
covariance-mediated information transfer, and preserves spatial and directional
structure in challenging high-dimensional regimes. The ablations further
isolate the contributions of cross-covariance transfer and nonlinear likelihood
correction. At PDE scale, FREESIA reduces RMSE by 56\% relative to the best
evaluated baseline in the sparse non-injective case. Future work targets
efficient sampling, scalable surrogates, and repeated-cycle theory under
stability conditions.
\endgroup

\clearpage
\subsection*{AI use statement}

Generative AI tools were used to assist with writing and language polishing,
drafting and revising parts of the manuscript, literature retrieval and
discovery, research ideation and methodological discussion, implementation and
debugging of research code, experimental design and analysis, and the
development and checking of mathematical claims and proof arguments. AI tools
were also used to assist with reference formatting and manuscript consistency
checks. All AI-assisted mathematical derivations, proof steps, implementation
code, experimental procedures, retrieved references, numerical results, and
manuscript text were independently reviewed, tested, or verified by the
authors. AI-generated or AI-edited code was inspected and validated against
the intended algorithms and experimental configurations before use, and
reported experimental results were generated from the authors' computational
experiments rather than accepted directly from AI-generated output. The
authors take full responsibility for the final methodology, implementation,
experiments, mathematical claims, and content of the paper.

\bibliography{iclr2026_conference}
\bibliographystyle{plainnat}

\appendix
\section{THEORETICAL DERIVATIONS AND PROOFS}
\label{app:proofs}

Throughout this appendix, we fix one assimilation cycle and suppress its index.

\subsection{Proof of Theorem 1}
\label{app:proof-posterior-flow}

For fixed \(y\), define the posterior-reweighted endpoint coupling
\begin{equation}
\mathbb P_y^a
:=
\mathcal L(Z_0,Z_1\mid Y=y).
\end{equation}
Since \(Z_t=(1-t)Z_0+tZ_1\), the surrogate posterior path satisfies
\begin{equation}
\widetilde p_t^a
=
(Z_t)_\#\mathbb P_y^a.
\end{equation}
For any \(\varphi\in C_c^1(\mathbb R^d)\), under the required first-moment condition,
\begin{equation}
\frac{d}{dt}\mathbb E_{\mathbb P_y^a}[\varphi(Z_t)]
=
\mathbb E_{\mathbb P_y^a}
\left[\nabla\varphi(Z_t)^\top(Z_1-Z_0)\right].
\end{equation}
Conditioning on \(Z_t\) shows that its marginal law satisfies the continuity equation in the weak sense with
\begin{equation}
v_t^a(z_t;y)
=
\mathbb E[Z_1-Z_0\mid Z_t=z_t,Y=y].
\end{equation}
Since \(Z_t=(1-t)Z_0+tZ_1\), for \(t<1\),
\begin{equation}
Z_1-Z_0=\frac{Z_1-Z_t}{1-t},
\end{equation}
and therefore
\begin{equation}
v_t^a(z_t;y)
=
\frac{\mathbb E[Z_1\mid Z_t=z_t,Y=y]-z_t}{1-t},
\end{equation}
which proves Theorem~\ref{thm:posterior-flow}.

\subsection{Analytic GMM bridge derivation}
\label{app:general-bridge}

For FREESIA, \(Z_0\sim\mathcal N(0,\Sigma^f)\), \(Z_1\mid J=j\sim\mathcal N(\mu_j,(1-\gamma^2)\Sigma^f)\), and \(Z_0\perp(Z_1,J)\). Therefore
\begin{equation}
\begin{aligned}
\operatorname{Cov}(Z_t\mid J=j)
&=[(1-t)^2+(1-\gamma^2)t^2]\Sigma^f,\\
\operatorname{Cov}(Z_1,Z_t\mid J=j)
&=t(1-\gamma^2)\Sigma^f.
\end{aligned}
\end{equation}
Applying the Gaussian conditioning formula gives
\begin{equation}
\begin{aligned}
\mu_{1\mid t,j}(z_t)
&=
\mu_j+
\frac{t(1-\gamma^2)}{(1-t)^2+(1-\gamma^2)t^2}
(z_t-t\mu_j),\\
\Sigma_{1\mid t}
&=
\frac{(1-\gamma^2)(1-t)^2}{(1-t)^2+(1-\gamma^2)t^2}
\Sigma^f.
\end{aligned}
\end{equation}
Bayes' rule for \(J\) yields the normalized responsibility
\begin{equation}
r_{J\mid Z_t}(j\mid z_t)
=
\frac{
\pi_j\mathcal N\!\left(
z_t;t\mu_j,[(1-t)^2+(1-\gamma^2)t^2]\Sigma^f
\right)
}{
\sum_\ell\pi_\ell\mathcal N\!\left(
z_t;t\mu_\ell,[(1-t)^2+(1-\gamma^2)t^2]\Sigma^f
\right)
}.
\end{equation}
These are the bridge quantities used in Section~\ref{sec:gmm-bridge}.

\subsection{Proof of Theorem 2 and affine closure}
\label{app:proof-nonlinear-correction}

For fixed \(z_t,y\), write \(r_j=r_{J\mid Z_t}(j\mid z_t)\) and
\(\phi_j(z_1)=\phi_{1\mid t,j}(z_1\mid z_t)\). The nonlinear augmented target satisfies
\begin{equation}
p(j,z_1\mid z_t,y)
\propto
r_j\phi_j(z_1)p(y\mid z_1),
\end{equation}
whereas the proposal satisfies
\begin{equation}
q(j,z_1\mid z_t,y)
\propto
r_j\phi_j(z_1)p_{\mathrm{lin}}(y\mid z_1,j,z_t).
\end{equation}
Consequently,
\begin{equation}
\frac{p(j,z_1\mid z_t,y)}{q(j,z_1\mid z_t,y)}
\propto
\frac{p(y\mid z_1)}{p_{\mathrm{lin}}(y\mid z_1,j,z_t)}.
\end{equation}
Taking the \(z_1\)-expectation under the normalized target gives the corrected endpoint-regression identity in Theorem~\ref{thm:nonlinear-regression}.

If \(h\) is affine, then \(p_{\mathrm{lin}}=p\), so \(q=p\) and the correction weights in Corollary~\ref{cor:self-normalized-regression} are constant.

\subsection{Stochastic Kalman realization}
\label{app:stochastic-kalman}

For the component-wise sampling rule in Eq.~\ref{eq:stochastic-proposal} of
Appendix~\ref{app:covariance-implementation},
the mean and covariance are
\begin{equation}
\begin{aligned}
\mathbb E[Z_1^{(s)}]
&=\mu_{1\mid t,j}+G_{t,j}[y-h(\mu_{1\mid t,j})]
=\mu^q_{1\mid t,j},\\
\operatorname{Cov}(Z_1^{(s)})
&=(I-G_{t,j}H_{t,j})\Sigma_{1\mid t}(I-G_{t,j}H_{t,j})^\top
+G_{t,j}RG_{t,j}^\top\\
&=\Sigma_{1\mid t}-G_{t,j}H_{t,j}\Sigma_{1\mid t}
=\Sigma^q_{1\mid t,j}.
\end{aligned}
\end{equation}
Thus \(Z_1^{(s)}\sim q_{Z_1\mid J,Z_t,Y}(\cdot\mid j,z_t,y)\).

\subsection{SNIS consistency and finite-sample velocity error}
\label{app:is-control}

Fix $t<1$, a flow state $z_t$, and the observation $y$. Let $P$ and $Q$
denote the exact augmented conditional target and the locally linearized
proposal from Eq.~\ref{eq:augmented-target-proposal}. As in Section~3, write
\[
g(j,z_1)
=
\frac{p(y\mid z_1)}
{p_{\mathrm{lin}}(y\mid z_1,j,z_t)}.
\]
Then
\[
\frac{dP}{dQ}(j,z_1)
=
\frac{g(j,z_1)}{\mathbb E_Q[g(J,Z_1)]}.
\]

For the finite-sample analysis, we estimate the velocity directly. Define the
local function
\begin{equation}
u_t(z_1;z_t)
:=
\frac{z_1-z_t}{1-t},
\qquad t<1.
\label{eq:rescaled-endpoint-increment}
\end{equation}
Along the straight path $Z_t=(1-t)Z_0+tZ_1$,
\[
u_t(Z_1;Z_t)=Z_1-Z_0.
\]
Consequently, Theorem~\ref{thm:posterior-flow} gives
\begin{equation}
v_t^a(z_t;y)
=
\frac{
\mathbb E_Q[g(J,Z_1)u_t(Z_1;z_t)]
}{
\mathbb E_Q[g(J,Z_1)]
}.
\label{eq:velocity-snis-identity}
\end{equation}
Moreover, the estimator in Eq.~\ref{eq:finite-endpoint-regression} is exactly
\begin{equation}
\widehat v^a_{t,N_{\mathrm{IS}}}(z_t;y)
=
\frac{
\sum_{s=1}^{N_{\mathrm{IS}}}
g(J^{(s)},Z_1^{(s)})u_t(Z_1^{(s)};z_t)
}{
\sum_{s=1}^{N_{\mathrm{IS}}}
g(J^{(s)},Z_1^{(s)})
}.
\label{eq:velocity-snis-estimator}
\end{equation}
Thus the practical \method{} velocity is a self-normalized
importance-sampling estimator of the conditional velocity.

\begin{proposition}[Automatic pointwise consistency of the corrected velocity estimator]
\label{prop:pointwise-snis}
For fixed $t<1$, $z_t$, and $y$, assume $R\succ0$, $0\le\gamma<1$,
$\Sigma^f\succ0$, and finite local linearization matrices $H_{t,j}$. Then
\[
0<\mathbb E_Q[g(J,Z_1)]<\infty,
\qquad
\mathbb E_Q\!\left[g(J,Z_1)\|u_t(Z_1;z_t)\|\right]<\infty.
\]
Consequently,
\[
\widehat v^a_{t,N_{\mathrm{IS}}}(z_t;y)
\longrightarrow v_t^a(z_t;y)
\quad\text{almost surely as }N_{\mathrm{IS}}\to\infty,
\]
or, equivalently,
$\widehat\mu^a_{1\mid t,N_{\mathrm{IS}}}(z_t;y)
\to\mu^a_{1\mid t}(z_t;y)$ almost surely.
\end{proposition}
\begin{proof}
Writing $\mathcal Z_q(t,z_t)>0$ for the proposal normalizer, the cancellation
between the proposal and likelihood ratio gives
\[
\mathbb E_Q[g]
=
\frac{1}{\mathcal Z_q}
\sum_j r_j\int\phi_j(z_1)p(y\mid z_1)\,dz_1.
\]
The Gaussian observation likelihood satisfies
\[
0<p(y\mid z_1)\le C_R,
\qquad
C_R:=(2\pi)^{-m/2}|R|^{-1/2},
\]
so the first moment is finite and strictly positive. Similarly,
\[
\mathbb E_Q[g\|u_t\|]
=
\frac{1}{\mathcal Z_q}
\sum_jr_j\int\phi_j(z_1)p(y\mid z_1)
\|u_t(z_1;z_t)\|\,dz_1
\]
is bounded above by
\[
\frac{C_R}{\mathcal Z_q}
\sum_jr_j\int\phi_j(z_1)\|u_t(z_1;z_t)\|\,dz_1,
\]
which is finite for every fixed $t<1$ because $\phi_j$ is Gaussian. The strong
law applies separately to the numerator and denominator of
Eq.~\ref{eq:velocity-snis-estimator}; ratio convergence proves the velocity
claim, and Eq.~\ref{eq:finite-endpoint-regression} gives the endpoint-mean claim.
\end{proof}

\paragraph{Pointwise bridge regularity.}
Set, only in this appendix,
\[
c_\gamma:=1-\gamma^2>0,
\qquad
D_t:=(1-t)^2+c_\gamma t^2.
\]
From Eq.~\ref{eq:conditional-gmm-decomposition},
\[
\Sigma_{1\mid t}
=
\frac{c_\gamma(1-t)^2}{D_t}\Sigma^f.
\]
Since
\[
\min_{t\in[0,1]}D_t
=
\frac{c_\gamma}{1+c_\gamma},
\]
we have
\begin{equation}
\frac{\Sigma_{1\mid t}}{(1-t)^2}
=
\frac{c_\gamma}{D_t}\Sigma^f
\preceq
(1+c_\gamma)\Sigma^f
\preceq
2\Sigma^f.
\label{eq:rescaled-bridge-cov-bound}
\end{equation}
The bridge mean satisfies
\[
\mu_{1\mid t,j}(z_t)
=
\mu_j+
\frac{tc_\gamma}{D_t}(z_t-t\mu_j),
\]
and direct algebra gives
\begin{equation}
\frac{\mu_{1\mid t,j}(z_t)-z_t}{1-t}
=
\frac{
(1-t)\mu_j+[t(1+c_\gamma)-1]z_t
}{
D_t
}.
\label{eq:rescaled-bridge-mean}
\end{equation}
For every fixed $0\le\gamma<1$, there is a constant
$C_\gamma<\infty$ such that
\begin{equation}
\left\|
\frac{\mu_{1\mid t,j}(z_t)-z_t}{1-t}
\right\|
\le
C_\gamma(1+\|\mu_j\|+\|z_t\|),
\qquad
0\le t<1.
\label{eq:rescaled-bridge-mean-growth}
\end{equation}
Thus, for every fixed $z_t\in\mathbb R^d$ and fixed
$0\le\gamma<1$, the rescaled bridge mean remains finite for $t<1$ and
converges to $z_t$ as $t\uparrow1$, while the rescaled bridge covariance
remains uniformly bounded in flow time. This regularity is pointwise in
$z_t$, with at most linear state dependence. No uniformity is claimed as
$\gamma\uparrow1$.

\paragraph{Endpoint behavior of the proposal.}
Assume $R\succeq r_0I$ for some $r_0>0$. We consider two sufficient regularity
regimes for the local observation model. In the smooth regime, $h$ is $C^2$ on
the relevant component neighborhoods, with finite bounds on $\|Dh\|$ and
$\|D^2h\|$. In the Lipschitz regime, $h$ is Lipschitz on the proposal states
under consideration,
\[
\|h(x)-h(x')\|\le L_h\|x-x'\|,
\]
and the local linearization matrix is a measurable selection $H(x)$ satisfying
\[
\|H(x)\|\le L_H.
\]
At differentiable points, $H(x)=Dh(x)$; at nondifferentiable points,
the implementation may use any bounded local linearization selection.
State-dependent versions of $L_h$ and $L_H$, as well as polynomial-growth
derivative bounds in the smooth regime, are admissible when the resulting
constants have the flow-law moments required below. The smooth regime gives a
quadratic local-linearization remainder, whereas the Lipschitz regime gives a
linear remainder. Both are sufficient for the endpoint likelihood-ratio
mismatch to vanish along the contracting bridge. For fixed $z_t$, let
$L_H(z_t)$ denote the corresponding finite bound on $\|H\|$ over the relevant
component neighborhoods. In the smooth regime, let $L_{h,2}(z_t)$ denote the
corresponding finite bound on $\|D^2h\|$.

The component-wise proposal in Eq.~\ref{eq:proposal-components} has mean and
covariance
\[
\mu^q_{1\mid t,j}
=
\mu_{1\mid t,j}
+
G_{t,j}[y-h(\mu_{1\mid t,j})],
\qquad
\Sigma^q_{1\mid t,j}
=
\Sigma_{1\mid t}
-
G_{t,j}H_{t,j}\Sigma_{1\mid t}.
\]
Because
\[
S_{t,j}
=
H_{t,j}\Sigma_{1\mid t}H_{t,j}^\top+R
\succeq R,
\]
we have $\|S_{t,j}^{-1}\|\le r_0^{-1}$ and hence
\begin{align}
\|G_{t,j}\|
&=
\|\Sigma_{1\mid t}H_{t,j}^\top S_{t,j}^{-1}\| \\
&\le
\frac{c_\gamma(1-t)^2}{D_t}
\|\Sigma^f\|L_H(z_t)r_0^{-1}
\le
C_G(z_t)(1-t)^2,
\label{eq:kalman-gain-endpoint-rate}
\end{align}
where $C_G(z_t)$ is finite for each fixed current state and independent of
$t$. Moreover,
\[
0\preceq\Sigma^q_{1\mid t,j}\preceq\Sigma_{1\mid t},
\]
so Eq.~\ref{eq:rescaled-bridge-cov-bound} implies
\begin{equation}
0
\preceq
\frac{\Sigma^q_{1\mid t,j}}{(1-t)^2}
\preceq
(1+c_\gamma)\Sigma^f
\preceq
2\Sigma^f.
\label{eq:proposal-rescaled-cov-bound}
\end{equation}
Finally,
\begin{equation}
\frac{\mu^q_{1\mid t,j}-z_t}{1-t}
=
\frac{\mu_{1\mid t,j}-z_t}{1-t}
+
\frac{G_{t,j}}{1-t}[y-h(\mu_{1\mid t,j})].
\label{eq:proposal-rescaled-mean}
\end{equation}
By Eq.~\ref{eq:kalman-gain-endpoint-rate}, the second term is $O(1-t)$
for each fixed $z_t$ whenever the observation residual is finite. Hence, for
every fixed finite $p$ for which the conditional residual moments are finite,
\[
\sup_{t_0(z_t)\le t<1}
\mathbb E_Q\!\left[
\|u_t(Z_1;z_t)\|^p
\right]
<
\infty
\]
for some $t_0(z_t)<1$.

For
\[
m:=\mu_{1\mid t,j}(z_t),
\qquad
H:=H_{t,j},
\qquad
\Delta z:=Z_1-m,
\]
the proposal satisfies
\[
\mathbb E_Q[\Delta z\mid J=j]
=
G_{t,j}[y-h(m)]
=
O((1-t)^2),
\]
and
\[
\operatorname{Cov}_Q(\Delta z\mid J=j)
=
\Sigma^q_{1\mid t,j}
=
O((1-t)^2).
\]
Gaussian moment bounds therefore give
$\Delta z=O_{L^p}(1-t)$ for every fixed finite $p$ and fixed state $z_t$ for
which the corresponding moments are finite.

In the smooth regime, Taylor's theorem gives
\[
h(m+\Delta z)
=
h(m)+H\Delta z+\rho_h(\Delta z),
\qquad
\|\rho_h(\Delta z)\|
\le
\frac{L_{h,2}(z_t)}2\|\Delta z\|^2,
\]
and hence $\rho_h(\Delta z)=O_{L^p}((1-t)^2)$. In the Lipschitz regime, define
the same local remainder by
\[
\rho_h(\Delta z)
:=
h(m+\Delta z)-h(m)-H\Delta z.
\]
The Lipschitz and bounded-linearization assumptions give
\[
\|\rho_h(\Delta z)\|
\le
\|h(m+\Delta z)-h(m)\|+\|H\Delta z\|
\le
(L_h+L_H)\|\Delta z\|.
\]
Thus, $\rho_h(\Delta z)=O_{L^p}(1-t)$. The remainder vanishes in both
regimes, quadratically in the smooth case and linearly in the Lipschitz case.

Writing
\[
a:=y-h(m)-H\Delta z,
\]
the Gaussian likelihoods give
\begin{equation}
\log g(J,Z_1)
=
a^\top R^{-1}\rho_h(\Delta z)
-
\frac12\rho_h(\Delta z)^\top R^{-1}\rho_h(\Delta z).
\label{eq:log-weight-endpoint}
\end{equation}
Under either regularity regime, $a=O_{L^p}(1)$ for every required finite $p$.
The smooth regime therefore gives
$\log g(J,Z_1)=O_{L^p}((1-t)^2)$, whereas the Lipschitz regime gives
$\log g(J,Z_1)=O_{L^p}(1-t)$. Hence, for each such fixed $z_t$,
\begin{equation}
g(J,Z_1)\to1
\quad\text{in $Q$-probability as }t\uparrow1.
\label{eq:weight-endpoint-consistency}
\end{equation}
The local-linearization mismatch therefore vanishes along the contracting
bridge. The nonsmooth observation maps used in the experiments fall within
the Lipschitz regime: $x\mapsto|x-c|$ and $x\mapsto|\arctan x|$ are globally
$1$-Lipschitz, while $x\mapsto|\arctan(10x)|$ is globally $10$-Lipschitz.
The branch derivative away from each kink and the bounded selection used at
the kink satisfy the required bound on $H$.

\paragraph{Endpoint importance moments at a fixed state.}
For fixed $t$ and $z_t$, define the normalizing constant of the locally
linearized proposal by
\[
\mathcal Z_q(t,z_t)
:=
\sum_{j=1}^M
r_j
\int
\phi_j(z_1)
p_{\mathrm{lin}}(y\mid z_1,j,z_t)
\,dz_1.
\]
Gaussian marginalization gives
\[
\mathcal Z_q(t,z_t)
=
\sum_{j=1}^M
r_j
\mathcal N\!\left(
y;
h(\mu_{1\mid t,j}(z_t)),
S_{t,j}(z_t)
\right).
\]
For each fixed $z_t$ and fixed $0\le\gamma<1$, continuity and positivity of
the Gaussian evidence, together with $S_{t,j}\to R$, imply that there are
$t_0(z_t)<1$ and $c_q(z_t)>0$ such that
\begin{equation}
\inf_{t_0(z_t)\le t<1}
\mathcal Z_q(t,z_t)
\ge
c_q(z_t)
>
0.
\label{eq:proposal-normalizer-lower-bound}
\end{equation}

Fix finite $r\ge1$ and $s\ge0$. The likelihood-ratio definition yields
\begin{align}
&\mathbb E_Q\!\left[
g(J,Z_1)^r
\left(1+\|u_t(Z_1;z_t)\|^s\right)
\right]
\nonumber\\
&=
\frac{1}{\mathcal Z_q(t,z_t)}
\sum_{j=1}^M
r_j
\int
\phi_j(z_1)
p(y\mid z_1)^r
p_{\mathrm{lin}}(y\mid z_1,j,z_t)^{1-r}
\nonumber\\
&\hspace{42mm}\times
\left(
1+
\left\|
\frac{z_1-z_t}{1-t}
\right\|^s
\right)
\,dz_1.
\label{eq:endpoint-moment-integral}
\end{align}
For bounded nonlinear likelihoods, a representative sufficient condition for
the Gaussian tail in component $j$ to have a finite $r$th weight moment is
\begin{equation}
(r-1)\lambda_{\max}\!\left(
\Sigma_{1\mid t}^{1/2}H_{t,j}^{\top}R^{-1}H_{t,j}
\Sigma_{1\mid t}^{1/2}
\right)<1.
\label{eq:snis-spectral-condition}
\end{equation}
The inverse linearized likelihood contributes the positive quadratic form
$(r-1)H_{t,j}^{\top}R^{-1}H_{t,j}$, whereas the Gaussian bridge contributes
$-\Sigma_{1\mid t}^{-1}$, both with the same factor $1/2$. This condition is
sufficient rather than necessary. The same condition applies in the Lipschitz
regime because it depends only on the bounded local linearization matrix $H$,
the bridge covariance, and $R$, rather than on second derivatives of $h$.
Since $\Sigma_{1\mid t}=O((1-t)^2)$, it becomes easier to satisfy near the
endpoint for every fixed finite $r$; small $R$, broad forecast uncertainty, or
early flow times can make complete-flow control more restrictive.
The exact Gaussian likelihood is uniformly bounded above. The inverse power
$p_{\mathrm{lin}}^{\,1-r}$ contributes an exponent that grows at most
quadratically in $z_1-\mu_{1\mid t,j}$. With the change of variables
\[
z_1
=
\mu_{1\mid t,j}
+
(1-t)\zeta,
\]
the bridge density becomes a Gaussian density in $\zeta$ with covariance
\[
\frac{\Sigma_{1\mid t}}{(1-t)^2}
=
\frac{c_\gamma}{D_t}\Sigma^f,
\]
which is uniformly bounded for fixed $\gamma<1$. The positive quadratic
coefficient contributed by $p_{\mathrm{lin}}^{\,1-r}$ is multiplied by
$(1-t)^2$. For fixed $z_t$, the relevant derivative and residual constants
are finite, so this perturbation is dominated by the negative Gaussian
quadratic form for all $t$ sufficiently close to $1$. The rescaled increment
becomes
\[
u_t(z_1;z_t)
=
\frac{\mu_{1\mid t,j}-z_t}{1-t}
+
\zeta,
\]
and its first term is uniformly bounded by
Eq.~\ref{eq:rescaled-bridge-mean-growth}.
Equation~\ref{eq:proposal-normalizer-lower-bound} then gives, after increasing
$t_0(z_t)$ if necessary,
\begin{equation}
\boxed{
\sup_{t_0(z_t)\le t<1}
\mathbb E_Q\!\left[
g(J,Z_1)^r
\left(1+\|u_t(Z_1;z_t)\|^s\right)
\right]
<
\infty.
}
\label{eq:endpoint-mixed-moment-bound}
\end{equation}
Thus, for each fixed state, the contracting bridge controls every fixed finite
importance-weight and rescaled-endpoint moment required below. Together with
Eq.~\ref{eq:weight-endpoint-consistency}, the same bound at a slightly higher
moment order gives uniform integrability in flow time and hence
$g(J,Z_1)\to1$ in $L^r(Q)$ for each fixed finite admissible $r$. In
particular, after increasing $t_0(z_t)$ if necessary,
\[
\inf_{t_0(z_t)\le t<1}
\mathbb E_Q[g(J,Z_1)]
>
0.
\]
Under either regularity regime, the bridge introduces no additional endpoint
singularity in the importance weights. Complete-flow behavior away from the endpoint is
handled below either directly from the observation model or through the general
accumulated-SNIS alternative.

\paragraph{Observation-model control in the bounded-residual regime.}
Assume
\begin{equation}
B_y:=\sup_{x\in\mathbb R^d}\|y-h(x)\|_{R^{-1}}<\infty,
\qquad
\Lambda_H:=\sup_x\lambda_{\max}\!\left(
(\Sigma^f)^{1/2}H(x)^\top R^{-1}H(x)(\Sigma^f)^{1/2}
\right)<\infty,
\label{eq:bounded-observation-constants}
\end{equation}
and define $\kappa:=(1-\gamma^2)\Lambda_H$. If
$\|H(x)\|\le L_H$, a simple sufficient bound is
\[
\Lambda_H\le\frac{L_H^2\|\Sigma^f\|}{\lambda_{\min}(R)},
\qquad
\kappa\le
\frac{(1-\gamma^2)L_H^2\|\Sigma^f\|}{\lambda_{\min}(R)}.
\]
These are coarse upper bounds, not identities.

The exact likelihood and its linearized counterpart satisfy
\[
p(y\mid z_1)
=C_R\exp\!\left[-\frac12\|y-h(z_1)\|_{R^{-1}}^2\right]
\ge C_Re^{-B_y^2/2},
\qquad
p_{\mathrm{lin}}(y\mid z_1,j,z_t)\le C_R.
\]
Consequently,
\begin{equation}
g(j,z_1)\ge g_\star:=e^{-B_y^2/2}>0.
\label{eq:importance-weight-lower-bound}
\end{equation}
Moreover, writing
\[
\Sigma_{1\mid t}=a_t\Sigma^f,
\qquad
a_t:=\frac{(1-\gamma^2)(1-t)^2}
{(1-t)^2+(1-\gamma^2)t^2}\le1-\gamma^2,
\]
the definition of $\kappa$ and $H_{t,j}=H(\mu_{1\mid t,j})$ give
\[
R^{-1/2}H_{t,j}\Sigma_{1\mid t}H_{t,j}^\top R^{-1/2}
\preceq\kappa I.
\]
Thus
\[
R\preceq S_{t,j}\preceq R^{1/2}(I+\kappa I)R^{1/2},
\qquad
|S_{t,j}|\le |R|(1+\kappa)^m,
\]
while $S_{t,j}^{-1}\preceq R^{-1}$ implies
\[
(y-h(\mu_{1\mid t,j}))^\top S_{t,j}^{-1}
(y-h(\mu_{1\mid t,j}))\le B_y^2.
\]
Every component evidence is therefore bounded below by
$C_R(1+\kappa)^{-m/2}e^{-B_y^2/2}$, and hence
\begin{equation}
\mathcal Z_q(t,z_t)\ge
Z_\star:=C_R(1+\kappa)^{-m/2}e^{-B_y^2/2}>0.
\label{eq:proposal-normalizer-uniform-lower-bound}
\end{equation}
The constant $Z_\star$ may depend strongly on the observation dimension $m$;
no dimension-free claim is made.

Let $p:=1+\delta$ with $0<\delta\le1$ and $\delta\kappa<1$. In
Eq.~\ref{eq:endpoint-moment-integral}, set
\[
m_j:=\mu_{1\mid t,j}(z_t),
\qquad
\Delta:=z_1-m_j,
\qquad
b_j:=y-h(m_j).
\]
The likelihood bounds give
\[
p(y\mid z_1)^p\le C_R^p,
\qquad
p_{\mathrm{lin}}^{\,1-p}
=C_R^{1-p}
\exp\!\left[
\frac{\delta}{2}(b_j-H_{t,j}\Delta)^\top
R^{-1}(b_j-H_{t,j}\Delta)
\right].
\]
The resulting Gaussian integral is finite whenever
\[
I-\delta\Sigma_{1\mid t}^{1/2}H_{t,j}^\top R^{-1}H_{t,j}
\Sigma_{1\mid t}^{1/2}\succ0.
\]
Its largest quadratic eigenvalue is at most $\delta\kappa<1$, uniformly in
$t$ and $j$. Completing the square, using $\|b_j\|_{R^{-1}}\le B_y$, and
applying the rescaling $\Delta=(1-t)\xi$ together with
Eqs.~\ref{eq:rescaled-bridge-cov-bound} and
\ref{eq:rescaled-bridge-mean-growth} shows that
\begin{equation}
\mathbb E_Q\!\left[
g^{1+\delta}\bigl(1+\|u_t\|^{1+\delta}\bigr)
\right]
\le C_\delta\bigl(1+\|z_t\|^{1+\delta}\bigr),
\qquad 0\le t<1.
\label{eq:bounded-observation-moment}
\end{equation}
Here the finite forecast ensemble ensures $\max_j\|\mu_j\|<\infty$ within
the assimilation cycle. The constant may depend on the fixed surrogate,
$y,R,\gamma,\Sigma^f,H,d,m$, and $\delta$, but not on $N_{\mathrm{IS}}$, $T$,
or $t$.

Set $N:=N_{\mathrm{IS}}$ and, for i.i.d. proposal samples, write
\[
g_s:=g(J^{(s)},Z_1^{(s)}),
\qquad
u_s:=u_t(Z_1^{(s)};z_t),
\qquad
v:=v_t^a(z_t;y),
\qquad
X_s:=g_s(u_s-v).
\]
Then $\mathbb E_Q[X_s]=0$ and
\[
\widehat v^a_{t,N}-v
=\frac{N^{-1}\sum_{s=1}^NX_s}{N^{-1}\sum_{s=1}^Ng_s}.
\]
Equation~\ref{eq:importance-weight-lower-bound} makes the denominator at least
$g_\star$. For $p=1+\delta\in(1,2]$, a componentwise von Bahr--Esseen or
Marcinkiewicz--Zygmund inequality gives
\[
\mathbb E\left\|N^{-1}\sum_{s=1}^NX_s\right\|^p
\le C_{p,d}N^{1-p}\mathbb E\|X_1\|^p.
\]
Equations~\ref{eq:importance-weight-lower-bound} and
\ref{eq:bounded-observation-moment} first imply
$\|v_t^a(z_t;y)\|\le C(1+\|z_t\|)$ and then
$\mathbb E_Q\|X_1\|^p\le C_p(1+\|z_t\|^p)$. Hence
\begin{equation}
\mathbb E\!\left[
\|\widehat v^a_{t,N}(z_t;y)-v_t^a(z_t;y)\|
\,\middle|\,z_t
\right]
\le C_\delta N^{-\beta_\delta}(1+\|z_t\|),
\qquad
\beta_\delta:=\frac{\delta}{1+\delta}.
\label{eq:bounded-snis-velocity}
\end{equation}
For $\delta=1$, this is the usual root-$N$ rate.

It remains to control evaluation at the random numerical states. Since the
normalized weights are nonnegative and sum to one,
\[
\|\widehat v\|^p\le\sum_{s=1}^Nw_s\|u_s\|^p.
\]
Using $\sum_sg_s\ge Ng_\star$ and
Eq.~\ref{eq:bounded-observation-moment} yields
\begin{equation}
\mathbb E[\|\widehat v_{t,N}(z_t)\|^p\mid z_t]
\le C_p(1+\|z_t\|^p).
\label{eq:numerical-velocity-moment}
\end{equation}
For the Euler recursion
$\widehat Z_{n+1}=\widehat Z_n+\Delta t\,
\widehat v_{t_n,N}(\widehat Z_n)$, convexity and
$\|x+y\|^p\le2^{p-1}(\|x\|^p+\|y\|^p)$ give
\[
\mathbb E\|\widehat Z_{n+1}\|^p
\le(1+C\Delta t)\mathbb E\|\widehat Z_n\|^p+C\Delta t.
\]
Because $Z_0$ is Gaussian, discrete Gr\"onwall therefore gives
\begin{equation}
\sup_{N\ge1,\,T\ge1}\max_{0\le n\le T}
\mathbb E\|\widehat Z_n\|^p\le C_Z<\infty.
\label{eq:numerical-flow-moment}
\end{equation}
Evaluating Eq.~\ref{eq:bounded-snis-velocity} at $z_t=\widehat Z_n$ now gives
\begin{equation}
a_n:=\mathbb E\|\widehat v_{t_n,N}(\widehat Z_n)
-v_{t_n}(\widehat Z_n)\|
\le C_\delta N^{-\beta_\delta},
\label{eq:bounded-numerical-velocity}
\end{equation}
and consequently
\begin{equation}
\Delta t\sum_{n=0}^{T-1}a_n
\le C_\delta N^{-\beta_\delta}.
\label{eq:bounded-accumulated-velocity}
\end{equation}

\paragraph{Affine observation operators.}
Suppose $h(x)=Ax+b$, with fixed $A$ and $b$. Then $H_{t,j}=A$ for every
component and flow state. Writing $m=\mu_{1\mid t,j}(z_t)$, we have
\[
h(m)+A(z_1-m)=Am+b+A(z_1-m)=Az_1+b=h(z_1).
\]
Thus the locally linearized likelihood is exact, even when $h$ is unbounded:
\begin{equation}
p_{\mathrm{lin}}(y\mid z_1,j,z_t)=p(y\mid z_1),
\qquad g(j,z_1)=1.
\label{eq:affine-unit-weight}
\end{equation}
The proposal $Q$ therefore equals the surrogate conditional target $P$, and
its independent endpoint samples give ordinary Monte Carlo estimators,
\begin{equation}
\widehat\mu^a_{1\mid t,N}(z_t;y)=\frac1N\sum_{s=1}^N Z_1^{(s)},
\qquad
\widehat v^a_{t,N}(z_t;y)
=\frac1N\sum_{s=1}^N u_t(Z_1^{(s)};z_t).
\label{eq:affine-monte-carlo-estimator}
\end{equation}

The bridge bounds in Eqs.~\ref{eq:rescaled-bridge-cov-bound} and
\ref{eq:rescaled-bridge-mean-growth} give
$\Sigma_{1\mid t}/(1-t)^2\preceq2\Sigma^f$ and
$\| (\mu_{1\mid t,j}(z_t)-z_t)/(1-t)\|
\le C_\gamma(1+\|\mu_j\|+\|z_t\|)$.
For the affine Gaussian proposal,
$0\preceq\Sigma^q_{1\mid t,j}\preceq\Sigma_{1\mid t}$ and
$G_{t,j}=\Sigma_{1\mid t}A^\top S_{t,j}^{-1}$, with
$S_{t,j}\succeq R$. Hence $\|G_{t,j}/(1-t)\|\le C(1-t)$.
Since the finite component means are fixed and
$\|y-Am-b\|\le C(1+\|z_t\|)$, the proposal mean shift divided by
$1-t$ is also at most linear in $\|z_t\|$. Averaging over the component
probabilities therefore yields, uniformly for $t<1$,
\begin{equation}
\mathbb E_Q\!\left[\|u_t(Z_1;z_t)\|^2\mid z_t\right]
\le C(1+\|z_t\|^2).
\label{eq:affine-velocity-second-moment}
\end{equation}
Because $v_t^a(z_t;y)=\mathbb E_Q[u_t(Z_1;z_t)\mid z_t]$, standard
Monte Carlo variance and Jensen's inequality give
\begin{equation}
\mathbb E\!\left[
\|\widehat v^a_{t,N}(z_t;y)-v_t^a(z_t;y)\|
\,\middle|\,z_t\right]
\le C N^{-1/2}(1+\|z_t\|).
\label{eq:affine-pointwise-velocity}
\end{equation}
Equation~\ref{eq:affine-velocity-second-moment} also bounds the second
moment of the sample-mean velocity. The random-state moment recursion used
in Eq.~\ref{eq:numerical-flow-moment}, now with $p=2$, gives
$\sup_{N,T}\max_n\mathbb E\|\widehat Z_n\|^2<\infty$.
Evaluating Eq.~\ref{eq:affine-pointwise-velocity} at $\widehat Z_n$ and
summing the Euler steps yields
\begin{equation}
\Delta t\sum_{n=0}^{T-1}
\mathbb E\!\left[
\|\widehat v^a_{t_n,N}(\widehat Z_n)
-v^a_{t_n}(\widehat Z_n)\|\right]
\le C N^{-1/2}.
\label{eq:affine-accumulated-velocity}
\end{equation}
This estimate does not require a bounded residual or an importance-weight
overlap condition.

\paragraph{General accumulated-SNIS alternative.}
The bounded nonlinear result uses the residual assumption $B_y<\infty$.
For more general unbounded nonlinear observation maps, we retain the following
accumulated condition on the standard pointwise SNIS constants. For each fixed
$(t,z_t)$ satisfying the standard SNIS moment
conditions, apply Theorem~2.3 of \citet{agapiou2017importance} with the
unbounded test function
\[
\varphi_\ell(j,z_1)
:=
u_t(z_1;z_t)_\ell.
\]
Let $C_{\mathrm{MSE},\ell}(t,z_t)$ denote the resulting finite
non-asymptotic MSE constant. Applying the theorem coordinatewise to
Eq.~\ref{eq:velocity-snis-estimator} gives the pointwise conditional estimate
\begin{equation}
\mathbb E\!\left[
\left|
\widehat v^a_{t,N_{\mathrm{IS}},\ell}(z_t;y)
-
v^a_{t,\ell}(z_t;y)
\right|^2
\,\middle|\, z_t
\right]
\le
\frac{C_{\mathrm{MSE},\ell}(t,z_t)}{N_{\mathrm{IS}}}.
\label{eq:pointwise-velocity-mse-bound}
\end{equation}

The pointwise estimate above does not require a uniform-in-time bound on the
SNIS constants. For the global finite-computation analysis, it is sufficient
to control their accumulated contribution along the numerical posterior flow.
Define
\[
\kappa_{\mathrm{IS}}(t,z)
:=
\left(
\sum_{\ell=1}^d C_{\mathrm{MSE},\ell}(t,z)
\right)^{1/2}.
\]
We assume that there exists $C_{\mathrm{IS}}<\infty$, independent of
$N_{\mathrm{IS}}$ and $T$, such that
\begin{equation}
\sup_{N_{\mathrm{IS}}\ge1,\;T\ge1}
\Delta t
\sum_{n=0}^{T-1}
\mathbb E\!\left[
\kappa_{\mathrm{IS}}(t_n,\widehat Z_n)
\right]
\le C_{\mathrm{IS}},
\qquad
\Delta t=\frac1T.
\label{eq:accumulated-is-constant}
\end{equation}
This condition controls the accumulated first-moment contribution of the
standard pointwise SNIS MSE constants along the numerical flow. It is weaker
than a uniform bound at every flow time: the latter implies
Eq.~\ref{eq:accumulated-is-constant} by Jensen's inequality, whereas the
accumulated condition allows the importance-sampling difficulty to vary with
$t$ provided that its total contribution remains finite.

Taking conditional expectations in
Eq.~\ref{eq:pointwise-velocity-mse-bound}, summing over coordinates, and
applying Jensen's inequality gives, at each numerical state,
\begin{equation}
\mathbb E\!\left[
\left\|
\widehat v^a_{t_n,N_{\mathrm{IS}}}(\widehat Z_n;y)
-
v^a_{t_n}(\widehat Z_n;y)
\right\|
\right]
\le
N_{\mathrm{IS}}^{-1/2}
\mathbb E\!\left[
\kappa_{\mathrm{IS}}(t_n,\widehat Z_n)
\right].
\label{eq:numerical-velocity-mean-bound}
\end{equation}
Therefore Eq.~\ref{eq:accumulated-is-constant} implies
\begin{equation}
\Delta t
\sum_{n=0}^{T-1}
\mathbb E\!\left[
\left\|
\widehat v^a_{t_n,N_{\mathrm{IS}}}(\widehat Z_n;y)
-
v^a_{t_n}(\widehat Z_n;y)
\right\|
\right]
\le
C_{\mathrm{IS}}N_{\mathrm{IS}}^{-1/2}.
\label{eq:accumulated-velocity-mean-bound}
\end{equation}

For any fixed $(t,z_t)$ satisfying the pointwise SNIS conditions,
Eqs.~\ref{eq:posterior_velocity} and
\ref{eq:finite-endpoint-regression} give
\[
\widehat\mu^a_{1\mid t,N_{\mathrm{IS}}}(z_t;y)
-
\mu^a_{1\mid t}(z_t;y)
=
(1-t)
\left(
\widehat v^a_{t,N_{\mathrm{IS}}}(z_t;y)
-
v_t^a(z_t;y)
\right).
\]
Hence Eq.~\ref{eq:pointwise-velocity-mse-bound} implies
\begin{equation}
\mathbb E\!\left[
\left\|
\widehat\mu^a_{1\mid t,N_{\mathrm{IS}}}(z_t;y)
-
\mu^a_{1\mid t}(z_t;y)
\right\|
\,\middle|\,z_t
\right]
\le
(1-t)N_{\mathrm{IS}}^{-1/2}
\kappa_{\mathrm{IS}}(t,z_t).
\label{eq:endpoint-regression-derived}
\end{equation}
Thus the endpoint-regression error retains the explicit $(1-t)$ contraction
factor without requiring a uniform-in-time SNIS constant. A continuous-time
analogue,
\[
\int_0^1
\mathbb E[\kappa_{\mathrm{IS}}(t,Z_t)]\,dt<\infty,
\]
similarly controls the accumulated SNIS error along the exact flow, but it is
not required for the finite Euler error bound below.

The conditional-velocity identity also gives the terminal value:
\[
v_t^a(z_t;y)
=
\mathbb E[Z_1-Z_0\mid Z_t=z_t,Y=y].
\]
At $t=1$, $Z_t=Z_1$. Since $Z_0$ is independent of $(Z_1,Y)$ and
$\mathbb E[Z_0]=0$,
\begin{equation}
v_1^a(z;y)
:=
\mathbb E[Z_1-Z_0\mid Z_1=z,Y=y]
=
z.
\label{eq:terminal-velocity-extension}
\end{equation}
Equations~\ref{eq:rescaled-bridge-mean}--%
\ref{eq:endpoint-mixed-moment-bound} give $v_t^a(z;y)\to z$ as
$t\uparrow1$ for every fixed $z$ covered by the pointwise assumptions. The
apparent endpoint singularity in Eq.~\ref{eq:posterior_velocity} is therefore
removable. The numerical method evaluates at $t_n=n/T<1$ and advances its
final Euler step to $t=1$.

Appendix~\ref{app:l96} reports finite-sample normalized-ESS diagnostics for the
sparse, small-noise Lorenz--96 regime. They show substantial weight
concentration over the early and intermediate flow and rapid recovery near the
endpoint, consistent with the endpoint contraction analysis. For general
unbounded nonlinear observations, complete-flow moment control remains
nontrivial and is supplied by the accumulated-SNIS alternative above.

\subsection{Finite-computation error}
\label{app:finite-computation}

Let
\[
\Delta t=\frac1T,
\qquad
t_n=n\Delta t=\frac nT,
\qquad
n=0,\ldots,T.
\]
Assume that the continuously extended ideal velocity admits a deterministic
Lipschitz constant $L_v<\infty$ on the supports relevant to the coupled exact
and numerical flow laws:
\[
\|v_t^a(x;y)-v_t^a(x';y)\|
\le
L_v\|x-x'\|
\]
for $t\in[0,1]$. Assume also that one Euler step for the exact ODE has mean
local truncation error bounded by $C_E(\Delta t)^2$, uniformly in $n$.

\begin{theorem}
\label{thm:finite-computation}
Under these flow regularity conditions, let $Z_t$ solve
\[
\dot Z_t=v_t^a(Z_t),
\qquad
Z_0\sim\mathcal L(Z_0),
\]
and let
\begin{equation}
\widehat Z_{n+1}
=
\widehat Z_n
+
\Delta t\,
\widehat v^a_{t_n,N_{\mathrm{IS}}}(\widehat Z_n),
\qquad
\widehat Z_0=Z_0.
\label{eq:numerical-flow-update}
\end{equation}
In the bounded nonlinear regime of Appendix~\ref{app:is-control}, for every
$0<\delta\le1$ with $\delta\kappa<1$,
\begin{equation}
W_1\!\left(
\widehat p^a_{N_{\mathrm{IS}},T},
\widetilde p^a
\right)
\le
e^{L_v}C_\delta N_{\mathrm{IS}}^{-\delta/(1+\delta)}
+
B(L_v)C_ET^{-1}.
\label{eq:finite-computation-bound}
\end{equation}
In the affine regime,
\begin{equation}
W_1\!\left(
\widehat p^a_{N_{\mathrm{IS}},T},
\widetilde p^a
\right)
\le
e^{L_v}C_1N_{\mathrm{IS}}^{-1/2}
+B(L_v)C_ET^{-1}.
\label{eq:affine-finite-computation-bound}
\end{equation}
Here
\[
B(L_v)
:=
\begin{cases}
(e^{L_v}-1)/L_v, & L_v>0,\\[1ex]
1, & L_v=0.
\end{cases}
\]
\end{theorem}

\begin{proof}
Couple the exact and numerical flows through the same initial variable $Z_0$,
and let
\[
e_n:=\mathbb E\|\widehat Z_n-Z_{t_n}\|.
\]
One exact ODE step satisfies
\[
Z_{t_{n+1}}
=
Z_{t_n}
+
\Delta t\,v^a_{t_n}(Z_{t_n})
+
\rho_n,
\qquad
\mathbb E\|\rho_n\|
\le
C_E(\Delta t)^2.
\]
Subtracting the numerical update, adding and subtracting
$v^a_{t_n}(\widehat Z_n)$, and using the $L_v$-Lipschitz property gives
\begin{equation}
e_{n+1}
\le
(1+L_v\Delta t)e_n
+
\Delta t\,a_n
+
C_E(\Delta t)^2,
\label{eq:error-recursion-full-interval}
\end{equation}
where
\[
a_n
:=
\mathbb E\left\|
\widehat v^a_{t_n,N_{\mathrm{IS}}}(\widehat Z_n)
-
v^a_{t_n}(\widehat Z_n)
\right\|.
\]
Equation~\ref{eq:bounded-numerical-velocity} bounds $a_n$ in the bounded
nonlinear regime. In the affine regime, Eq.~\ref{eq:affine-pointwise-velocity}
and the random-state second-moment bound give the corresponding root-$N$
estimate. Thus $a_n\le r_N$, where
\[
r_N:=
\begin{cases}
C_\delta N_{\mathrm{IS}}^{-\delta/(1+\delta)},
&\text{bounded nonlinear},\\
C_1N_{\mathrm{IS}}^{-1/2},
&\text{affine}.
\end{cases}
\]
Writing $q=1+L_v\Delta t$ and iterating the recursion from $e_0=0$,
\[
e_T
\le
\Delta t\sum_{n=0}^{T-1}q^{T-1-n}a_n
+
C_E(\Delta t)^2\sum_{n=0}^{T-1}q^{T-1-n}.
\]
Since $q^T\le e^{L_v}$,
\[
\Delta t\sum_{n=0}^{T-1}q^{T-1-n}a_n
\le
e^{L_v}r_N.
\]
Moreover,
\[
C_E(\Delta t)^2\sum_{k=0}^{T-1}q^k
=
C_E\Delta t\,\frac{q^T-1}{L_v}
\le
B(L_v)C_ET^{-1},
\]
with the continuous interpretation $B(0)=1$. Therefore
\[
e_T
\le
e^{L_v}r_N+B(L_v)C_ET^{-1}.
\]
The coupling has terminal marginals
$\widehat p^a_{N_{\mathrm{IS}},T}$ and $\widetilde p^a$, hence
\[
W_1\!\left(
\widehat p^a_{N_{\mathrm{IS}},T},
\widetilde p^a
\right)
\le
\mathbb E\|\widehat Z_T-Z_1\|
=
e_T,
\]
which proves Eqs.~\ref{eq:finite-computation-bound} and
\ref{eq:affine-finite-computation-bound}.
\end{proof}
For more general unbounded nonlinear observation maps, the same recursion
combined with the accumulated-SNIS condition of Appendix~\ref{app:is-control}
recovers the corresponding general finite-computation bound.

\subsection{Stability of Bayesian reweighting}
\label{app:bayes-stability}

For this proof, write \(\ell_y(x)=p(y\mid x)\) and define the Bayesian
reweighting operator
\begin{equation}
\mathcal B_y(\mu)(dx)
:=
\frac{\ell_y(x)\mu(dx)}{\mu(\ell_y)}.
\end{equation}

\begin{theorem}
\label{thm:bayes-stability}
Let \(\mu,\nu\in\mathcal P_2(\mathbb R^d)\). Assume \(0\le\ell_y\le M_L\),
\(\ell_y\) is \(G_L\)-Lipschitz, and define
\begin{equation}
\begin{aligned}
\mathcal Z_\mu&:=\mu(\ell_y)>0,
\qquad
\mathcal Z_\nu:=\nu(\ell_y)>0,\\
\mathcal Z_\star&:=\min\{\mathcal Z_\mu,\mathcal Z_\nu\},
\qquad
B_2:=
\max_{\rho\in\{\mu,\nu\}}
\left(\int\|x\|^2\rho(dx)\right)^{1/2}.
\end{aligned}
\end{equation}
Then
\begin{equation}
W_1(\mathcal B_y(\mu),\mathcal B_y(\nu))
\le
C_{\mathrm{Bayes}}(y)W_2(\mu,\nu),
\end{equation}
where
\begin{equation}
C_{\mathrm{Bayes}}(y)
=
\frac{M_L+G_LB_2}{\mathcal Z_\star}
+
\frac{M_LG_LB_2}{\mathcal Z_\star^2}.
\end{equation}
\end{theorem}

\begin{proof}
By Kantorovich--Rubinstein duality, consider a 1-Lipschitz \(f\) with \(f(0)=0\), so \(|f(x)|\le\|x\|\). Decompose
\begin{equation}
\mathcal B_y(\mu)(f)-\mathcal B_y(\nu)(f)
=
\frac{\mu(f\ell_y)-\nu(f\ell_y)}{\mathcal Z_\mu}
+
\nu(f\ell_y)\left(\frac1{\mathcal Z_\mu}-\frac1{\mathcal Z_\nu}\right).
\end{equation}
For any \(\Gamma\in\Pi(\mu,\nu)\) and \((U,V)\sim\Gamma\), Cauchy--Schwarz gives
\begin{equation}
|\mu(f\ell_y)-\nu(f\ell_y)|
\le
(M_L+G_LB_2)
\left(\mathbb E_\Gamma\|U-V\|^2\right)^{1/2}.
\end{equation}
Moreover, \(|\mathcal Z_\mu-\mathcal Z_\nu|\le G_LW_2(\mu,\nu)\) and
\(|\nu(f\ell_y)|\le M_LB_2\). Taking the infimum over
\(\Gamma\in\Pi(\mu,\nu)\), then the supremum over \(f\), yields the stated
constant.
\end{proof}

\subsection{Forecast and posterior surrogate error}
\label{app:gmm-smoothing}
\label{app:posterior-surrogate}

Let
\[
\nu_M^f
=
\frac1M\sum_{j=1}^M\delta_{x_j^f}
\]
denote the equal-weight empirical forecast measure. Its mean is
\[
\bar x^f
=
\frac1M\sum_{j=1}^M x_j^f,
\]
and its covariance is
\begin{equation}
C_{\nu_M}^f
:=
\operatorname{Cov}(\nu_M^f)
=
\frac1M\sum_{j=1}^M
(x_j^f-\bar x^f)(x_j^f-\bar x^f)^\top.
\label{eq:empirical-covariance}
\end{equation}

\begin{lemma}
\label{lem:gmm-forecast}
Let \(U\sim\nu_M^f\) and \(\xi\sim\mathcal N(0,I)\) independently, with \(\Sigma^f\succeq0\). The FREESIA forecast-GMM surrogate is the law of
\begin{equation}
\widetilde U
=
\bar x^f+\gamma(U-\bar x^f)
+\sqrt{1-\gamma^2}(\Sigma^f)^{1/2}\xi.
\end{equation}
Then
\begin{equation}
\begin{aligned}
W_2(\widetilde p^f,p^f)
\le\;&
W_2(\nu_M^f,p^f)
+\sqrt{2(1-\gamma)\operatorname{tr}C_{\nu_M}^f}\\
&+
\sqrt{1-\gamma^2}
\left\|
(\Sigma^f)^{1/2}-(C_{\nu_M}^f)^{1/2}
\right\|_F.
\end{aligned}
\end{equation}
\end{lemma}

\begin{proof}
Using the same \((U,\xi)\), compare \(\widetilde U\) with
\begin{equation}
U_C
=
\bar x^f+\gamma(U-\bar x^f)
+\sqrt{1-\gamma^2}(C_{\nu_M}^f)^{1/2}\xi.
\end{equation}
The two couplings satisfy
\begin{equation}
\mathbb E\|U_C-U\|^2
=2(1-\gamma)\operatorname{tr}C_{\nu_M}^f,
\qquad
\mathbb E\|\widetilde U-U_C\|^2
=(1-\gamma^2)
\left\|(\Sigma^f)^{1/2}-(C_{\nu_M}^f)^{1/2}\right\|_F^2.
\end{equation}
The triangle inequality with \(W_2(\nu_M^f,p^f)\) proves the bound. Independence also gives
\begin{equation}
\mathbb E\|\widetilde U\|^2
=
\|\bar x^f\|^2
+\gamma^2\operatorname{tr}C_{\nu_M}^f
+(1-\gamma^2)\operatorname{tr}\Sigma^f,
\end{equation}
so \(\widetilde p^f\in\mathcal P_2(\mathbb R^d)\).
\end{proof}

We leave \(W_2(\nu_M^f,p^f)\) explicit rather than impose a dimension-free empirical-measure rate, since generic Wasserstein empirical convergence depends on dimension and moment assumptions \citep{fournier2023empirical}.

Recall that $\varepsilon_{\mathrm{sur}}(y)=W_1(\widetilde p^a,p^a)$ from
Eq.~\ref{eq:surrogate-error}.

\begin{corollary}
\label{cor:posterior-surrogate}
Under the conditions of Theorem~\ref{thm:bayes-stability} for \((p^f,\widetilde p^f)\),
\begin{equation}
\begin{aligned}
\varepsilon_{\mathrm{sur}}(y)
\le\;&
C_{\mathrm{Bayes}}(y)
\Bigg[
W_2(\nu_M^f,p^f)
+\sqrt{2(1-\gamma)\operatorname{tr}C_{\nu_M}^f}\\
&\hspace{25mm}
+\sqrt{1-\gamma^2}
\left\|
(\Sigma^f)^{1/2}-(C_{\nu_M}^f)^{1/2}
\right\|_F
\Bigg].
\end{aligned}
\end{equation}
\end{corollary}

\begin{proof}
Apply Theorem~\ref{thm:bayes-stability} to \((p^f,\widetilde p^f)\) and substitute Lemma~\ref{lem:gmm-forecast}.
\end{proof}

The three contributions in this bound are finite-ensemble discrepancy, GMM
smoothing, and covariance modification, respectively.

\subsection{Proof of Theorem 3}
\label{app:proof-total-error}

By the triangle inequality and Eq.~\ref{eq:error-split},
\[
W_1(\widehat p^a_{N_{\mathrm{IS}},T},p^a)
\le
W_1(\widehat p^a_{N_{\mathrm{IS}},T},\widetilde p^a)
+\varepsilon_{\mathrm{sur}}(y).
\]
Applying the corresponding case of Theorem~\ref{thm:finite-computation} gives
\[
W_1(\widehat p^a_{N_{\mathrm{IS}},T},p^a)
\le
\begin{cases}
C_\delta N_{\mathrm{IS}}^{-\delta/(1+\delta)}
+C_2T^{-1}+\varepsilon_{\mathrm{sur}}(y),
&\text{bounded nonlinear},\\
C_1N_{\mathrm{IS}}^{-1/2}
+C_2T^{-1}+\varepsilon_{\mathrm{sur}}(y),
&\text{affine},
\end{cases}
\]
with constants relabeled as in Theorem~\ref{thm:total-error}.
When $\kappa<1$, $\delta=1$ is admissible in the bounded nonlinear case and
yields the stated corollary.

\subsection{Covariance and high-dimensional implementation}
\label{app:covariance-implementation}
\label{app:high-dimensional}
\label{app:optional-approximations}

\paragraph{Complete finite-ensemble update.}
Algorithm~\ref{alg:freesia} gives the full-evidence \method{} update; the
high-dimensional experiments use the diagonal categorical-evidence
approximation described below. It implements the three stages derived in
Section~\ref{sec:method}:
observation-adaptive component sampling, covariance-aware endpoint sampling,
and nonlinear correction with a flow update.
In Algorithm~\ref{alg:freesia}, $z_t^{(i)}$ denotes a deterministic numerical
flow particle, whereas $\zeta_1^{(i,s)}$ denotes its $s$th stochastic endpoint
proposal.

\begin{algorithm}[H]
\caption{\method{}: Training-Free Ensemble Posterior Flow}
\label{alg:freesia}
\scriptsize
\begin{algorithmic}[1]
\Require Forecast ensemble \(\mathcal E^f=\{x_j^f\}_{j=1}^M\), observation \(y\),
observation operator \(h\), observation-noise covariance \(R\), GMM separation
\(\gamma\), flow steps \(T\), and endpoint importance samples
\(N_{\mathrm{IS}}\)
\State Construct the GMM component means \(\{\mu_j\}_{j=1}^M\) and the
localized/inflated covariance \(\Sigma^f=L_fL_f^\top\)
\State Set \(\Delta t=1/T\)
\State Draw \(z_0^{(i)}=L_f\xi_0^{(i)}\),
\(\xi_0^{(i)}\sim\mathcal N(0,I)\), for \(i=1,\ldots,M\)
\For{\(n=0,\ldots,T-1\)}
  \State \(t\gets n\Delta t\)
  \For{\(i=1,\ldots,M\)}
    \State Compute
    \(\{r_{J\mid Z_t}(j\mid z_t^{(i)}),\mu_{1\mid t,j}(z_t^{(i)})\}_{j=1}^M\)
    and \(\Sigma_{1\mid t}\) using
    Eq.~\ref{eq:conditional-gmm-decomposition}
    \State Compute
    \(\{H_{t,j}^{(i)},S_{t,j}^{(i)},q_{J\mid Z_t,Y}(j\mid z_t^{(i)},y),G_{t,j}^{(i)}\}_{j=1}^M\)
    using Eqs.~\ref{eq:proposal-parameters}--\ref{eq:proposal-components}
    \For{\(s=1,\ldots,N_{\mathrm{IS}}\)}
      \State \textbf{Step 1: Observation-adaptive component sampling.}
      Draw \(j^{(i,s)}\sim q_{J\mid Z_t,Y}(\cdot\mid z_t^{(i)},y)\)
      \State \textbf{Step 2: Covariance-aware endpoint sampling.}
      Draw \(\zeta_1^{(i,s)}\sim q_{Z_1\mid J,Z_t,Y}
      (\cdot\mid j^{(i,s)},z_t^{(i)},y)\)
      using Eq.~\ref{eq:stochastic-proposal} in the sampling paragraph below
    \EndFor
    \State \textbf{Step 3: Nonlinear correction and flow update.} Compute
    \(\widetilde w^{(i,s)}=p(y\mid \zeta_1^{(i,s)})/
    p_{\rm lin}(y\mid \zeta_1^{(i,s)},j^{(i,s)},z_t^{(i)})\)
    for \(s=1,\ldots,N_{\mathrm{IS}}\), and normalize
    \(w^{(i,s)}=\widetilde w^{(i,s)}/
    \sum_{r=1}^{N_{\mathrm{IS}}}\widetilde w^{(i,r)}\)
    \State Set \(\widehat\mu_{1\mid t}^a=
    \sum_{s=1}^{N_{\mathrm{IS}}}w^{(i,s)}\zeta_1^{(i,s)}\) and
    \(\widehat v_t^a=(\widehat\mu_{1\mid t}^a-z_t^{(i)})/(1-t)\)
    \State Update
    \(z_{t+\Delta t}^{(i)}=z_t^{(i)}+\Delta t\,\widehat v_t^a\)
  \EndFor
\EndFor
\State Set $x^{a,(i)}\gets z_1^{(i)}$ for $i=1,\ldots,M$
\State \Return $\mathcal E^a=\{x^{a,(i)}\}_{i=1}^{M}$
\end{algorithmic}
\end{algorithm}

\paragraph{Sampling and numerical integration.}
Factor \(\Sigma^f=L_fL_f^\top\) once per assimilation cycle. For a sampled
component \(J^{(s)}=j\), draw \(\xi^{(s)}\sim\mathcal N(0,I)\) and
\(\epsilon^{(s)}\sim\mathcal N(0,R)\) independently and use
\begin{equation}
\label{eq:stochastic-proposal}
\begin{aligned}
\widetilde Z_1^{(s)}
&=
\mu_{1\mid t,j}
+
\sqrt{\frac{(1-\gamma^2)(1-t)^2}{(1-t)^2+(1-\gamma^2)t^2}}\,
L_f\xi^{(s)},\\
Z_1^{(s)}
&=
\widetilde Z_1^{(s)}+
G_{t,j}\!\left[
y+\epsilon^{(s)}-h(\mu_{1\mid t,j})
-H_{t,j}(\widetilde Z_1^{(s)}-\mu_{1\mid t,j})
\right].
\end{aligned}
\end{equation}
Conditional on \(J^{(s)}=j\), this samples exactly from
\(q_{Z_1\mid J,Z_t,Y}(\cdot\mid j,z_t,y)\); see
Appendix~\ref{app:stochastic-kalman}.

All likelihood-ratio weights are evaluated in the log domain,
\begin{equation}
\log\widetilde w^{(s)}
=\log p(y\mid Z_1^{(s)})
-\log p_{\mathrm{lin}}(y\mid Z_1^{(s)},J^{(s)},z_t),
\label{eq:log-importance-weight}
\end{equation}
and normalized by subtracting a log-sum-exp term (implemented as a stable
softmax) to avoid numerical underflow or overflow in high-dimensional
observation spaces.

With $t_n=n/T$ and $\Delta t=1/T$, each flow particle is advanced by
$z^{(i)}_{t_{n+1}}=z^{(i)}_{t_n}+\Delta t\,
\hat v^a_{t_n,N_{\mathrm{IS}}}(z^{(i)}_{t_n};y)$.
Although the endpoint-regression form in
Eq.~\ref{eq:finite-endpoint-regression} is written for $t<1$,
Appendix~\ref{app:is-control} shows that its apparent endpoint singularity is
removable under the stated regularity conditions. Algorithm~\ref{alg:freesia}
evaluates the velocity only at $t_n=n/T<1$ and advances the final Euler step
directly to $t=1$, exactly as analyzed in
Appendix~\ref{app:finite-computation}.

After $T$ flow steps, set $x^{a,(i)}:=z_1^{(i)}$ and
$\mathcal E^a:=\{x^{a,(i)}\}_{i=1}^M=\{z_1^{(i)}\}_{i=1}^M$.

Two covariance matrices play distinct roles. Recall from
Appendix~\ref{app:posterior-surrogate} that $C_{\nu_M}^f$ denotes the
covariance of the equal-weight empirical forecast measure. The standard
uniform-ensemble sample covariance is
\begin{equation}
\widehat C_M^f
:=
\frac1{M-1}
\sum_{j=1}^M
(x_j^f-\bar x^f)(x_j^f-\bar x^f)^\top.
\end{equation}
The factor \(1/M\) in \(C_{\nu_M}^f\) arises because it is the covariance of
the equal-weight empirical probability measure, whereas \(\widehat C_M^f\)
uses the standard unbiased \(1/(M-1)\) ensemble covariance employed in the
practical localized/inflated update.
The practical covariance used by FREESIA is
\begin{equation}
\Sigma^f
=
\lambda_{\mathrm{infl}}
\bigl(\Lambda\odot\widehat C_M^f\bigr).
\end{equation}
The covariance-modification term in Corollary~\ref{cor:posterior-surrogate}
measures the discrepancy between these choices. The density and responsibility
formulas assume \(\Sigma^f\succ0\), whereas the surrogate Wasserstein bound in
Lemma~\ref{lem:gmm-forecast} only requires \(\Sigma^f\succeq0\). In
implementation, a small numerical jitter is added when necessary to obtain a
positive-definite factorization.

The same factor $L_f$ generates source and endpoint noise and evaluates bridge
responsibilities. Because the determinant is shared across components,
\begin{equation}
\log r_j
=
c(z_t,t)
+\log\pi_j
-
\frac{
\|L_f^{-1}(z_t-t\mu_j)\|^2
}{
2[(1-t)^2+(1-\gamma^2)t^2]
},
\end{equation}
where \(c(z_t,t)\) is independent of the component label \(j\).

\paragraph{Diagonal evidence approximation.}
\label{app:diagonal-evidence}
For large observation dimension, component-label adaptation may use \(S_{t,j}\mapsto\operatorname{Diag}(\operatorname{diag}S_{t,j})\). With
\begin{equation}
e_j=\mathcal N(y;h(\mu_{1\mid t,j}),S_{t,j}),
\qquad
\widetilde e_j=
\mathcal N\!\left(
y;h(\mu_{1\mid t,j}),
\operatorname{Diag}(\operatorname{diag}S_{t,j})
\right),
\end{equation}
let \(q_j=r_je_j/A\), \(\widetilde q_j=r_j\widetilde e_j/\widetilde A\), \(A=\sum_jr_je_j\), and \(\widetilde A=\sum_jr_j\widetilde e_j\). Then
\begin{equation}
\label{eq:diagonal-categorical-local-bound}
\|q-\widetilde q\|_1
\le
\frac{2}{A}
\sum_{j=1}^M r_j|e_j-\widetilde e_j|.
\end{equation}
This diagonal component-evidence approximation is used in all reported
Lorenz--96 and Kolmogorov-flow \method{} experiments. It changes only the
categorical component-adaptation probabilities; the conditional endpoint
proposal in Eqs.~\ref{eq:proposal-parameters}--\ref{eq:proposal-components} is
otherwise unchanged. The nonlinear likelihood-ratio correction alone does not
remove the categorical approximation introduced by replacing \(e_j\) with
\(\widetilde e_j\). Its induced population-flow error is analyzed in
Appendix~\ref{app:diagonal-evidence-error}.

\paragraph{Block-structured forecast surrogate for Kolmogorov flow.}
For the \(128\times128\) Kolmogorov-flow experiments, the practical
forecast-surrogate covariance is represented by a block-structured
approximation over the full \(d=32768\) state. The spatial domain is partitioned
into 16 non-overlapping blocks indexed by \(b=1,\ldots,16\), and the localized
covariance is factorized independently within each block:
\begin{equation}
\label{eq:kolmogorov-block-covariance}
\begin{aligned}
\Sigma^f_{\mathrm{blk}}
&=
\operatorname{blkdiag}
\!\left(
\Sigma^f_1,\ldots,\Sigma^f_{16}
\right),
&
L_{\mathrm{blk}}
&=
\operatorname{blkdiag}
\!\left(
L_1,\ldots,L_{16}
\right),
&
\Sigma^f_b&=L_bL_b^\top .
\end{aligned}
\end{equation}
For this benchmark, \(\Sigma^f_{\mathrm{blk}}\) is the covariance defining the
practical forecast surrogate, rather than an auxiliary covariance used only for
noise generation. The source law, bridge covariance, component
responsibilities, covariance--observation products, \(S_{t,j}\), \(G_{t,j}\),
and conditional endpoint sampling all use this same covariance. Accordingly,
the equations in Section~\ref{sec:method} and Algorithm~\ref{alg:freesia} are
interpreted with \(\Sigma^f=\Sigma^f_{\mathrm{blk}}\) and
\(L_f=L_{\mathrm{blk}}\) for this experiment. Matrix-free ensemble products
evaluate actions of this covariance without introducing a second covariance
model. The discrepancy between \(\Sigma^f_{\mathrm{blk}}\) and the empirical
forecast covariance is included in the covariance-modification contribution in
Corollary~\ref{cor:posterior-surrogate}. The nonlinear likelihood-ratio
correction is evaluated globally over the complete observed vector for each
endpoint proposal, rather than normalized independently by block.

\paragraph{Computational complexity.}
Let $M$, $T$, $d$, and $m$ denote the ensemble size, number of generative
steps, state dimension, and observation dimension, respectively. We separate the cost
of one component-conditioned proposal into $C_{\mathrm{cond}}$ for constructing
the observation-space conditional quantities, including
$S_{t,j}=H_{t,j}\Sigma_{1\mid t}H_{t,j}^{\top}+R$ and its linear-solve
factorization; $C_{\mathrm{end}}$ for drawing one conditional endpoint with
Eq.~\ref{eq:stochastic-proposal} after these quantities are available; and
$C_{\mathrm{like}}$ for evaluating the nonlinear likelihood-ratio correction.
The leading analysis work is therefore
\begin{equation}
O\!\left(TM^2C_{\mathrm{cond}}+
TMN_{\mathrm{IS}}(C_{\mathrm{end}}+C_{\mathrm{like}})\right).
\label{eq:general-analysis-complexity}
\end{equation}
For structured covariance and sampling-factor actions that cost $O(d)$,
component-wise or row-sparse observation Jacobians, and a direct dense solve
in observation space, $m$ covariance--observation products and factorization
of the $m\times m$ matrix $S_{t,j}$ give
$C_{\mathrm{cond}}=O(dm+m^3)$. Once this factorization is available, a
structured noise draw, observation-space triangular solve, and covariance
back-projection give $C_{\mathrm{end}}=O(d+m^2)$. For component-wise nonlinear
observations and diagonal $R$, as in the reported high-dimensional
experiments, $C_{\mathrm{like}}=O(m)$. With $m\le d$, a representative
structured realization thus has
\begin{equation}
O\!\left[TM^2(dm+m^3)+TMN_{\mathrm{IS}}(d+m^2)\right]
\label{eq:structured-analysis-complexity}
\end{equation}
analysis work, up to lower-order terms and method-dependent Jacobian costs.
A matrix-free realization avoids storing the full $d\times m$
state--observation covariance and uses
$O(Md+S_{\Sigma}+m^2)$ memory, where $S_{\Sigma}$ is structured-covariance
storage; explicitly storing $\Sigma H^{\top}$ adds $O(dm)$ memory. Dense
observation Jacobians or alternative linear solvers change these
$m$-dependent costs. Appendix~\ref{app:highdim-scaling} studies the diagonal
specialization that removes the dense observation-space algebra.

\subsection{Error analysis for the diagonal component-evidence approximation}
\label{app:diagonal-evidence-error}

Fix \(t<1\), a flow state \(z_t\), and an observation \(y\). Write
\(r_j:=r_{J\mid Z_t}(j\mid z_t)\) and
\(\phi_j(z_1):=\phi_{1\mid t,j}(z_1\mid z_t)\), and abbreviate
\(S_j:=S_{t,j}\). Define the diagonal evidence covariance and the corresponding
component probabilities by
\begin{equation}
\label{eq:diagonal-evidence-definitions}
\begin{aligned}
D_j&:=\operatorname{Diag}(\operatorname{diag}S_j),\\
e_j
&:=\int \phi_j(z_1)
p_{\mathrm{lin}}(y\mid z_1,j,z_t)\,dz_1
=\mathcal N\!\left(y;h(\mu_{1\mid t,j}),S_j\right),\\
\widetilde e_j
&:=\mathcal N\!\left(y;h(\mu_{1\mid t,j}),D_j\right),\\
q_j&:=\frac{r_je_j}{\sum_k r_ke_k},
\qquad
\widetilde q_j:=\frac{r_j\widetilde e_j}{\sum_k r_k\widetilde e_k}.
\end{aligned}
\end{equation}
The conditional endpoint density \(q(z_1\mid j)\) remains the Gaussian
conditional in Eq.~\ref{eq:proposal-components}; only the categorical component
probabilities are approximated.

\paragraph{Diagonal evidence and corrected component weights.}
Let
\begin{equation}
\label{eq:nonlinear-component-laws}
a_j:=\int \phi_j(z_1)p(y\mid z_1)\,dz_1,
\qquad
\pi_j^a(dz_1):=
\frac{\phi_j(z_1)p(y\mid z_1)}{a_j}\,dz_1.
\end{equation}
The full-evidence corrected endpoint law can be written as
\begin{equation}
\label{eq:full-corrected-mixture}
P^a(dz_1)=\sum_{j=1}^M\alpha_j\pi_j^a(dz_1),
\qquad
\alpha_j:=\frac{r_ja_j}{\sum_k r_ka_k}.
\end{equation}
For the diagonal categorical proposal, the conditional density and nonlinear
correction are
\[
q(z_1\mid j)
=\frac{\phi_j(z_1)p_{\mathrm{lin}}(y\mid z_1,j,z_t)}{e_j},
\qquad
g(j,z_1)=\frac{p(y\mid z_1)}
{p_{\mathrm{lin}}(y\mid z_1,j,z_t)}.
\]
Consequently,
\begin{equation}
\label{eq:diagonal-corrected-factor}
\begin{aligned}
\widetilde q_jq(z_1\mid j)g(j,z_1)
&\propto
r_j\widetilde e_j
\frac{\phi_j(z_1)p_{\mathrm{lin}}(y\mid z_1,j,z_t)}{e_j}
\frac{p(y\mid z_1)}{p_{\mathrm{lin}}(y\mid z_1,j,z_t)}\\
&=r_j\frac{\widetilde e_j}{e_j}\phi_j(z_1)p(y\mid z_1).
\end{aligned}
\end{equation}
With \(c_j:=\widetilde e_j/e_j\), the resulting population law is
\begin{equation}
\label{eq:diagonal-corrected-mixture}
P_{\mathrm{diag}}^a(dz_1)
=\sum_{j=1}^M\widetilde\alpha_j\pi_j^a(dz_1),
\qquad
\widetilde\alpha_j
:=\frac{r_ja_jc_j}{\sum_k r_ka_kc_k}
=\frac{\alpha_jc_j}{\sum_k\alpha_kc_k}.
\end{equation}
Thus diagonal evidence leaves the corrected within-component posterior laws
\(\pi_j^a\) unchanged and perturbs only their mixture weights.

\paragraph{Evidence and component-weight error.}
Let \(d_j:=y-h(\mu_{1\mid t,j})\) and
\(\eta_j:=\log(\widetilde e_j/e_j)\). Directly from the two Gaussian
evidences,
\begin{equation}
\label{eq:diagonal-log-evidence-error}
\eta_j
=\frac12\left[
\log\frac{|S_j|}{|D_j|}
+d_j^\top(S_j^{-1}-D_j^{-1})d_j
\right].
\end{equation}
Define the whitened off-diagonal perturbation and residual by
\begin{equation}
\label{eq:diagonal-whitened-perturbation}
F_j:=D_j^{-1/2}(S_j-D_j)D_j^{-1/2},
\qquad
\xi_j:=D_j^{-1/2}d_j.
\end{equation}
Since \(S_j=D_j^{1/2}(I+F_j)D_j^{1/2}\),
Eq.~\ref{eq:diagonal-log-evidence-error} is equivalently
\begin{equation}
\label{eq:diagonal-whitened-evidence-error}
\eta_j=\frac12\left[
\log\det(I+F_j)
+\xi_j^\top\bigl((I+F_j)^{-1}-I\bigr)\xi_j
\right].
\end{equation}
Thus the error is determined by omitted observation-space correlations through
a log-determinant term and a correlated Mahalanobis-residual term. If
\(\rho_j:=\|F_j\|_2<1\), then \(F_j\) has zero diagonal and
\(\operatorname{tr}(F_j)=0\). The eigenvalue expansion of
\(\log\det(I+F_j)\) and the resolvent identity give
\begin{equation}
\label{eq:diagonal-evidence-operator-bounds}
\left|\log\det(I+F_j)\right|
\le\frac{\|F_j\|_F^2}{2(1-\rho_j)},
\qquad
\left\|(I+F_j)^{-1}-I\right\|_2
\le\frac{\rho_j}{1-\rho_j}.
\end{equation}
Therefore
\begin{equation}
\label{eq:diagonal-log-evidence-bound}
|\eta_j|
\le
\frac{\|F_j\|_F^2}{4(1-\rho_j)}
+\frac{\rho_j}{2(1-\rho_j)}\|\xi_j\|^2.
\end{equation}

The relevant perturbation measure is the oscillation
\begin{equation}
\label{eq:evidence-error-oscillation}
\omega_e(t,z_t):=\max_j\eta_j-\min_j\eta_j,
\end{equation}
because a common additive shift in the log evidences cancels under categorical
normalization. Centering the \(\eta_j\) values within their range and applying
the bounded density-ratio inequality to the normalized exponential tilt in
Eq.~\ref{eq:diagonal-corrected-mixture} yields
\begin{equation}
\label{eq:diagonal-component-weight-bound}
\|\widetilde\alpha-\alpha\|_1
\le 2\tanh\!\left(\frac{\omega_e}{2}\right).
\end{equation}

\paragraph{Endpoint, velocity, and terminal-law error.}
Define
\begin{equation}
\label{eq:diagonal-component-means}
m_j^a(t,z_t;y):=\int z_1\pi_j^a(dz_1),
\qquad
M_a(t,z_t):=\max_j\|m_j^a-\mu_{1\mid t}^a\|.
\end{equation}
Then
\(\mu_{1\mid t}^a=\sum_j\alpha_jm_j^a\) and
\(\mu_{1\mid t}^{a,\mathrm{diag}}=\sum_j\widetilde\alpha_jm_j^a\).
Using \(\sum_j(\widetilde\alpha_j-\alpha_j)=0\) together with
Eq.~\ref{eq:diagonal-component-weight-bound} gives
\begin{equation}
\label{eq:diagonal-endpoint-mean-bound}
\|\mu_{1\mid t}^{a,\mathrm{diag}}-\mu_{1\mid t}^a\|
\le
2M_a(t,z_t)\tanh\!\left(\frac{\omega_e(t,z_t)}{2}\right).
\end{equation}
Thus the corresponding population velocities satisfy
\begin{equation}
\label{eq:diagonal-velocity-bound}
\|v_t^{a,\mathrm{diag}}-v_t^a\|
\le
\Delta_{\mathrm{diag}}(t,z_t)
:=
\frac{2M_a(t,z_t)}{1-t}
\tanh\!\left(\frac{\omega_e(t,z_t)}{2}\right).
\end{equation}

For the high-dimensional experiments, \(R\) is diagonal. The off-diagonal part
of \(S_{t,j}=H_{t,j}\Sigma_{1\mid t}H_{t,j}^\top+R\) is therefore generated
entirely by the bridge covariance term. Here the Kolmogorov-flow bridge uses the
same block-structured forecast-surrogate covariance defined in
Eq.~\ref{eq:kolmogorov-block-covariance}. Since
\(\Sigma_{1\mid t}=O((1-t)^2)\), bounded \(H_{t,j}\) gives
\(S_j-D_j=O((1-t)^2)\). Moreover,
\(D_j\succeq R\succeq r_0I\), so
\(\rho_j=O((1-t)^2)\). Under bounded normalized-residual moments,
Eqs.~\ref{eq:diagonal-log-evidence-bound}--\ref{eq:evidence-error-oscillation}
give \(\omega_e=O_{L^p}((1-t)^2)\). If \(M_a(t,z_t)\) remains finite near the
endpoint, then
\begin{equation}
\label{eq:diagonal-endpoint-contraction}
\Delta_{\mathrm{diag}}(t,z_t)=O_{L^p}(1-t).
\end{equation}
The diagonal component-evidence approximation therefore introduces no additional
\(1/(1-t)\)-type endpoint singularity in the population velocity.

Let \(Z_t\) and \(Z_t^{\mathrm{diag}}\) solve the full- and diagonal-evidence
population ODEs from the same source variable, and let \(p_{\mathrm{diag}}^a\)
denote the terminal law of the latter. If the full velocity is
\(L_v\)-Lipschitz in the state, Gr\"onwall's inequality under this synchronous
coupling gives
\begin{equation}
\label{eq:diagonal-terminal-error}
\begin{aligned}
W_1(p_{\mathrm{diag}}^a,\widetilde p^a)
&\le \varepsilon_{\mathrm{diag}}(y),\\
\varepsilon_{\mathrm{diag}}(y)
&:=\int_0^1 e^{L_v(1-t)}
\mathbb E\!\left[
\frac{2M_a(t,Z_t^{\mathrm{diag}})}{1-t}
\tanh\!\left(
\frac{\omega_e(t,Z_t^{\mathrm{diag}})}{2}
\right)
\right]dt.
\end{aligned}
\end{equation}
Under the same moment conditions, the endpoint contraction in
Eq.~\ref{eq:diagonal-endpoint-contraction} makes the integrand locally
integrable near \(t=1\).

Let $\widehat p^a_{\mathrm{diag},N_{\mathrm{IS}},T}$ denote the law produced by
the corresponding finite diagonal-evidence update with $N_{\mathrm{IS}}$
endpoint proposals per velocity evaluation and $T$ Euler steps.

Under the bounded-residual nonlinear conditions of
Appendix~\ref{app:is-control} and the corresponding numerical-flow regularity
conditions, the same finite-computation analysis applies to the
diagonal-evidence population flow. For any $0<\delta\le1$ satisfying
$\delta\kappa<1$,
\begin{equation}
\label{eq:diagonal-practical-error}
W_1\!\left(
\widehat p^a_{\mathrm{diag},N_{\mathrm{IS}},T},p^a
\right)
\le
C_\delta^{\mathrm{diag}}N_{\mathrm{IS}}^{-\delta/(1+\delta)}
+C_2^{\mathrm{diag}}T^{-1}
+\varepsilon_{\mathrm{diag}}(y)
+\varepsilon_{\mathrm{sur}}(y).
\end{equation}
If $\kappa<1$, the finite-sampling term is
$O(N_{\mathrm{IS}}^{-1/2})$. For affine observations, the local likelihood
correction has $g\equiv1$ by Eq.~\ref{eq:affine-unit-weight}. The diagonal
categorical evidence can still differ from the full evidence, but the same
ordinary Monte Carlo argument gives
\begin{equation}
\label{eq:affine-diagonal-practical-error}
W_1\!\left(
\widehat p^a_{\mathrm{diag},N_{\mathrm{IS}},T},p^a
\right)
\le
C_1^{\mathrm{diag}}N_{\mathrm{IS}}^{-1/2}
+C_2^{\mathrm{diag}}T^{-1}
+\varepsilon_{\mathrm{diag}}(y)
+\varepsilon_{\mathrm{sur}}(y).
\end{equation}
For more general unbounded nonlinear maps, the accumulated-SNIS alternative
applies under its stated condition.
The four terms are, respectively, finite endpoint sampling, flow
discretization, diagonal component-evidence approximation, and forecast-surrogate
error. Equation~\ref{eq:diagonal-categorical-local-bound} quantifies the local
source of the diagonal approximation, but by itself does not imply a
dimension-independent terminal bound on \(\varepsilon_{\mathrm{diag}}(y)\).

\paragraph{Collapsed-Gaussian special case.}
When \(\gamma=0\), all component means and bridge distributions coincide. The
full and diagonal evidences are then each independent of the component label,
so \(c_j=\widetilde e_j/e_j\) is common to all components and cancels under
categorical normalization. Consequently,
\begin{equation}
\label{eq:diagonal-collapsed-gaussian}
\gamma=0
\quad\Longrightarrow\quad
\widetilde\alpha_j=\alpha_j,
\quad
\omega_e=0,
\quad
\varepsilon_{\mathrm{diag}}(y)=0.
\end{equation}
This exact special case applies to the stride-2 and stride-4 Lorenz--96 regimes,
which use \(\gamma=0\). The dense Lorenz--96 and Kolmogorov-flow experiments use
\(\gamma=0.3\) and \(\gamma=0.2\), respectively, so no zero-error conclusion is
made for those regimes.

\section{Experimental Details}
\label{app:experiments}

This appendix summarizes the dynamical systems, initialization protocols,
and numerical settings used in the three benchmarks. Within each repetition,
all compared methods use the same truth trajectory and observation locations.
The Double-Well, Lorenz--96, and Kolmogorov-flow comparisons average
\(N_{\rm run}=10\) independent trajectories and all analysis times.

\paragraph{Baseline configuration.}
The numerical comparisons include EnKF-MDA
\citep{emerick2012enkfmda}, EnKF \citep{evensen2003enkf}, LETKF
\citep{hunt2007letkf}, EnSF \citep{bao2024ensf}, and the EnFF variants
\citep{transue2025enff}.
Hyperparameters were selected using reduced-cost pilot runs on trajectories
generated independently from those used in the final reported evaluation.
Within the final evaluation set, all compared methods use matched truth
trajectories and observation locations. Tables~\ref{tab:l96_loc}
and~\ref{tab:kf_loc} report covariance localization and inflation settings only
for \method{} and the Kalman-type baselines. EnSF and EnFF use their
method-specific sampling/guidance hyperparameters, which are reported separately
where relevant. Method-specific sampling budgets are specified in
Sections~B.1--B.3.

EnSF-LR is not included in the numerical comparison. Its
observed-to-unobserved update relies on ensemble-covariance linear regression
\citep{xiong2026ensflr}, while the released implementation does not include
localization for this step. In our high-dimensional small-ensemble settings,
the observed-state sample covariance is necessarily rank deficient, so
extending EnSF-LR would require additional localization or regularization
choices beyond the released formulation. We therefore discuss it as closely
related work without introducing a separately modified baseline.

IEnSF is a closely related covariance-aware extension of EnSF. We implemented
its published formulation but could not independently validate the resulting
performance against a reference implementation; we therefore omit numerical
results and restrict the comparison to methodological differences supported by
the published formulation. The experiments reported here should not be
interpreted as an empirical comparison with IEnSF.

\subsection{Double-Well}
\label{app:doublewell}

Following \citet{ding2026ssls}, the double-well benchmark is
\[
X_{k+1}
=
X_k-\delta t\,U'(X_k)
+\beta\sqrt{\delta t}\,V_k,
\qquad
U(x)=x^4-2x^2,
\]
where \(V_k\sim\mathcal N(0,1)\), \(\delta t=0.1\), and \(\beta=0.8\).
The initial filtering distribution is \(X_0\sim\mathcal N(0,0.5^2)\).
Observations are generated by
\[
Y_k=|X_k-0.4|+\epsilon_k,
\qquad
\epsilon_k\sim\mathcal N(0,0.1^2).
\]

The experiment uses 100 analysis steps and an ensemble size of 100.
A 10,000-point grid on \([-2.5,2.5]\) provides the numerical Bayes reference.
\method{} uses \(\gamma=0.9\), 20 flow steps, and \(N_{\mathrm{IS}}=10\)
endpoint importance samples.

For a density $p$ and observation kink $x_0=0.4$, the implemented left and
right mode masses are
\[
P_L(p)=\int_{x<x_0}p(x)\,dx,
\qquad
P_R(p)=\int_{x\ge x_0}p(x)\,dx.
\]
With $p_{\rm ref}$ denoting the grid Bayes posterior, the reported mode-mass
error is
\begin{equation}
E_{\rm mode}(p,p_{\rm ref})
=\frac12\left(
|P_L(p)-P_L(p_{\rm ref})|+|P_R(p)-P_R(p_{\rm ref})|
\right).
\label{eq:mode-mass-error}
\end{equation}
This posterior-shape metric is distinct from RMSE and CRPS, which evaluate
state or predictive accuracy.

\begin{figure}[!htbp]
\centering
\includegraphics[width=0.66\linewidth]{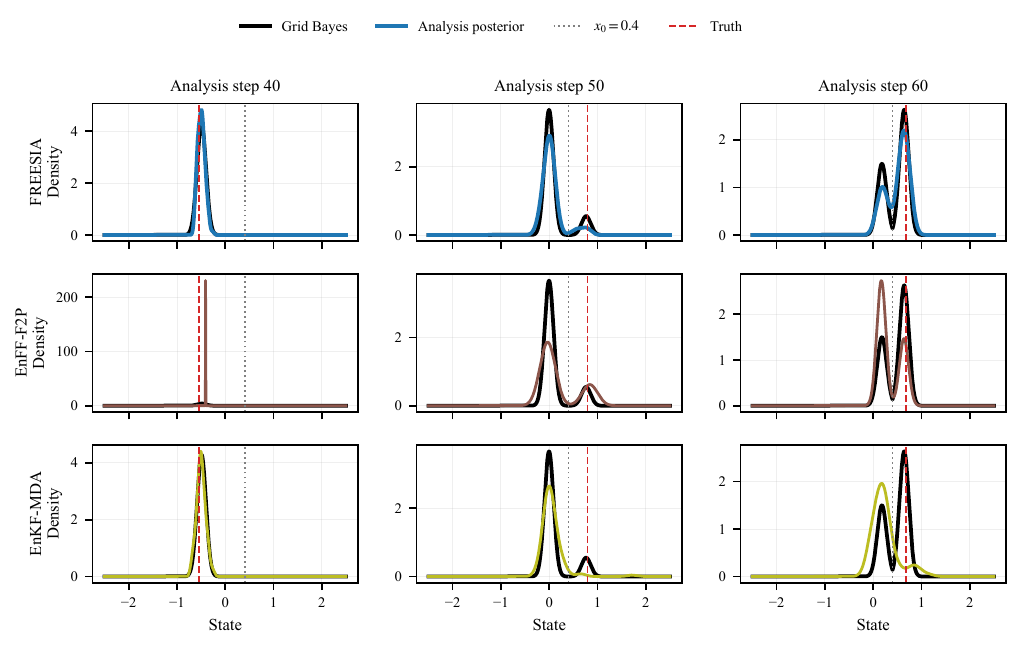}
\caption{Double-well posterior densities at analysis steps 40, 50, and 60. Rows show FREESIA, EnFF-F2P, and EnKF-MDA against the grid Bayes posterior.}
\label{fig:double_well_posterior}
\end{figure}

Figure~\ref{fig:double-well-grid-search} shows that separating the GMM
components is important in this multimodal regime: performance improves
markedly from the collapsed Gaussian limit, while about 20 flow steps already
capture most of the gain.

\begin{figure}[!htbp]
\centering
\includegraphics[width=1.00\linewidth]{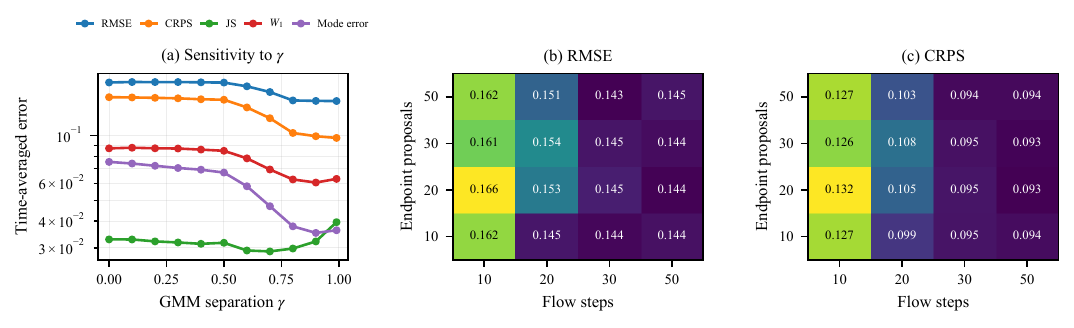}
\caption{Double-well hyperparameter search with $N_{\rm run}=3$ per setting. Left: all five reported errors versus the GMM separation parameter $\gamma$, with a logarithmic vertical axis; the rightmost point is $\gamma=0.99$. Center and right: RMSE and CRPS over flow steps and endpoint importance samples $N_{\mathrm{IS}}$. Lower is better throughout; values average all search runs and analysis times.}
\label{fig:double-well-grid-search}
\end{figure}
\begin{figure}[!htbp]
\centering
\includegraphics[width=0.98\linewidth]{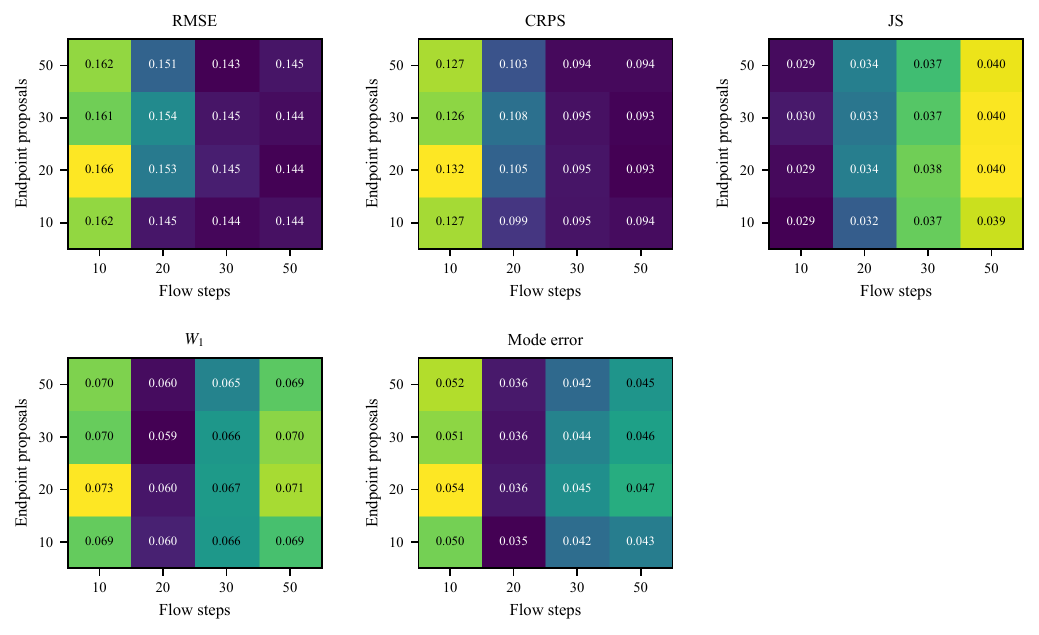}
\caption{Complete double-well hyperparameter search over flow steps and endpoint importance samples \(N_{\mathrm{IS}}\). The panels report RMSE, CRPS, Jensen--Shannon divergence, Wasserstein-1 distance, and mode-mass error; lower is better for every metric. Cell values are averaged over the hyperparameter-search repetitions and analysis times.}
\label{fig:doublewell_hyperparameter_full}
\end{figure}
\FloatBarrier

\subsection{Lorenz--96}
\label{app:l96}

We use the Lorenz--96 system \citep{lorenz2006predictability}, whose state
\(X=(x_1,\ldots,x_d)^\top\) evolves according to
\[
\frac{dx_i}{dt}
=
(x_{i+1}-x_{i-2})x_{i-1}
-x_i
+F,
\qquad
i=1,\ldots,d,
\]
with cyclic indexing, \(F=8\), and \(d=1000\).
Before assimilation, the system is integrated for 10 physical time units so
that initialization is taken from the chaotic attractor. Truth trajectories
and initial forecast ensembles are drawn from this post-spin-up regime rather
than formed by small perturbations of the truth. The spin-up interval is
excluded from all reported filtering statistics.

The truth model uses \(\Delta t_{\rm truth}=0.01\), while all forecast models
use \(\Delta t_{\rm model}=0.02\). Observations arrive every 0.2 time units,
with an ensemble size of 20 and 101 analysis times. We consider three
observation regimes: dense \(|\arctan(x_i)|\) observations with noise standard
deviation 0.05, stride-2 \(\arctan(x_{2i})\) observations with noise standard
deviation 0.05, and stride-4 \(\arctan(x_{4i})\) observations with noise
standard deviation 0.01.
The stride-4 experiment is a high-SNR sparse-observation stress test. Its
directly observed coordinates are strongly constrained, making the experiment
diagnostic of information transfer to the substantially larger unobserved
subset; the smaller observation covariance also produces a sharper likelihood
for guidance and posterior correction.

All \method{} configurations in the main Lorenz--96 comparison use 20 flow steps and
\(N_{\mathrm{IS}}=5\) endpoint importance samples.
EnKF-MDA uses 20 analysis substeps, EnSF uses 500 reverse-SDE steps, and each
EnFF variant uses 20 path grid points. These are the computational budgets
used in the comparison rather than matched-cost settings. The dense \method{}
experiment uses \(\gamma=0.3\), whereas both sparse regimes use the collapsed
single-Gaussian limit \(\gamma=0\).

\paragraph{Localization and inflation.}
Table~\ref{tab:l96_loc} reports the localization radius and inflation used in
the three Lorenz--96 regimes. Radius is measured in cyclic Lorenz--96 grid
points. \method{}, EnKF, EnKF-MDA, and LETKF use cyclic Gaspari--Cohn
localization through their respective covariance or ensemble updates.

\begin{table}[!htbp]
\centering
\scriptsize
\caption{Localization radius \(r_{\mathrm{loc}}\) and inflation parameter \(\lambda_{\mathrm{infl}}\) in the three Lorenz--96 experiments. GC denotes cyclic Gaspari--Cohn localization.}
\label{tab:l96_loc}
\begin{tabular}{lrrrrrr}
\toprule
Method & \multicolumn{2}{c}{Dense abs-atan} & \multicolumn{2}{c}{Stride-2 atan} & \multicolumn{2}{c}{Stride-4 atan} \\
\cmidrule(lr){2-3}\cmidrule(lr){4-5}\cmidrule(lr){6-7}
& \(r_{\mathrm{loc}}\) & \(\lambda_{\mathrm{infl}}\) & \(r_{\mathrm{loc}}\) & \(\lambda_{\mathrm{infl}}\) & \(r_{\mathrm{loc}}\) & \(\lambda_{\mathrm{infl}}\) \\
\midrule
\method{} & 4 & 1.8 & 4 & 1.2 & 2 & 1.0 \\
EnKF & 2 & 1.1 & 4 & 1.2 & 3 & 1.1 \\
EnKF-MDA & 3 & 1.6 & 3 & 1.1 & 2 & 1.2 \\
LETKF & 2 & 1.2 & 2 & 1.5 & 2 & 1.3 \\
\bottomrule
\end{tabular}
\end{table}
\FloatBarrier

\paragraph{Supplementary diagnostics.}
The observed/unobserved RMSE decomposition provides additional diagnostics for
the two sparse-observation regimes. Representative coordinate-wise trajectories
are reported for the dense and sparse regimes.

\begin{figure}[!htbp]
\centering
\includegraphics[width=0.92\linewidth]{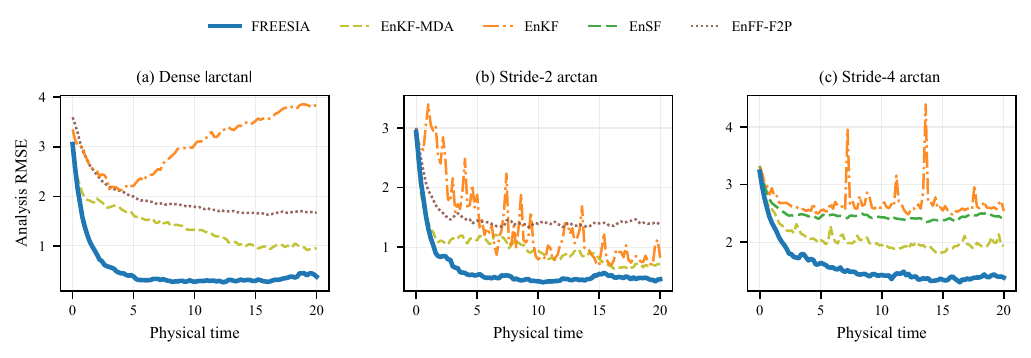}
\caption{Total analysis RMSE for dense absolute-arctangent, stride-2 arctangent, and stride-4 arctangent observations. Each panel shows the four methods with the lowest time-averaged RMSE in that regime; the shared legend lists the union of methods shown across panels.}
\label{fig:l96-rmse}
\end{figure}
\begin{figure}[!htbp]
\centering
\includegraphics[width=0.88\linewidth]{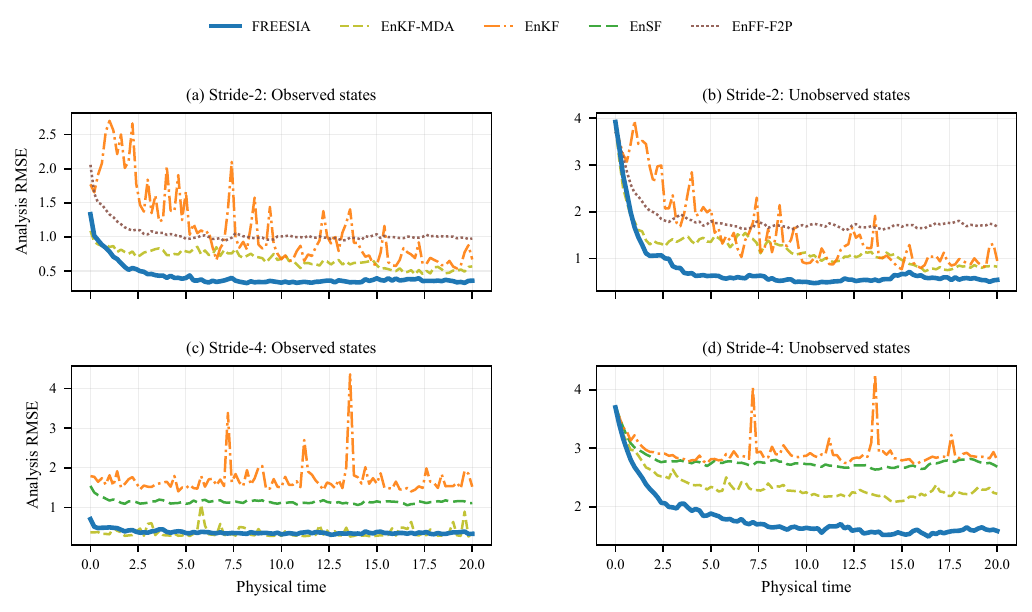}
\caption{Observed- and unobserved-coordinate RMSE for stride-2 (top) and stride-4 (bottom) observations. The right column directly evaluates information transfer to unobserved state coordinates.}
\label{fig:l96_obs_unobs}
\end{figure}

\begin{figure}[tbp]
\centering
\includegraphics[width=\linewidth]{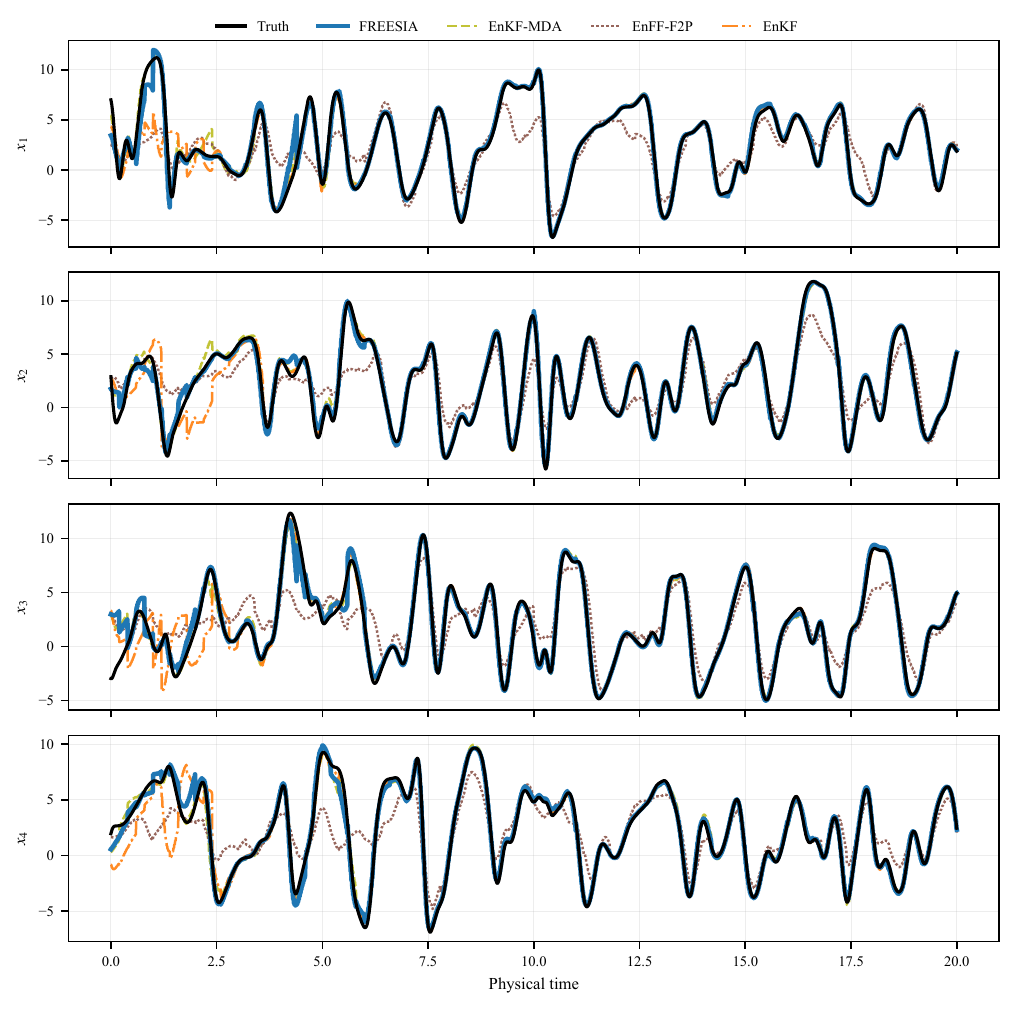}
\caption{Representative trajectories for the first four Lorenz--96 state coordinates under dense absolute-arctangent observations. The panels show the truth and analysis trajectories for the methods selected by the corresponding RMSE comparison.}
\label{fig:l96_dense_absatan_traj}
\end{figure}

\begin{figure}[tbp]
\centering
\includegraphics[width=\linewidth]{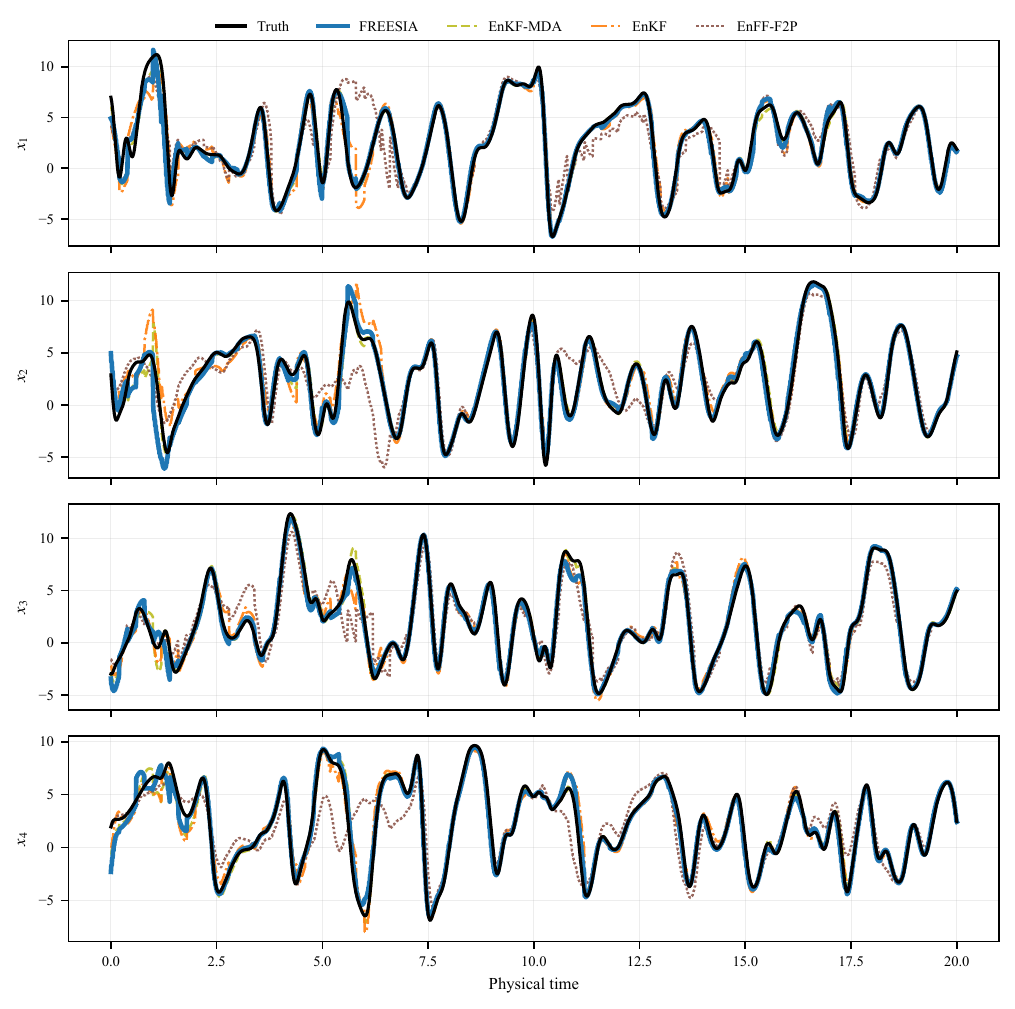}
\caption{Representative trajectories for the first four Lorenz--96 state coordinates under stride-2 arctangent observations. The panels show the truth and analysis trajectories for the methods selected by the corresponding RMSE comparison.}
\label{fig:l96_stride2_traj}
\end{figure}

\begin{figure}[tbp]
\centering
\includegraphics[width=\linewidth]{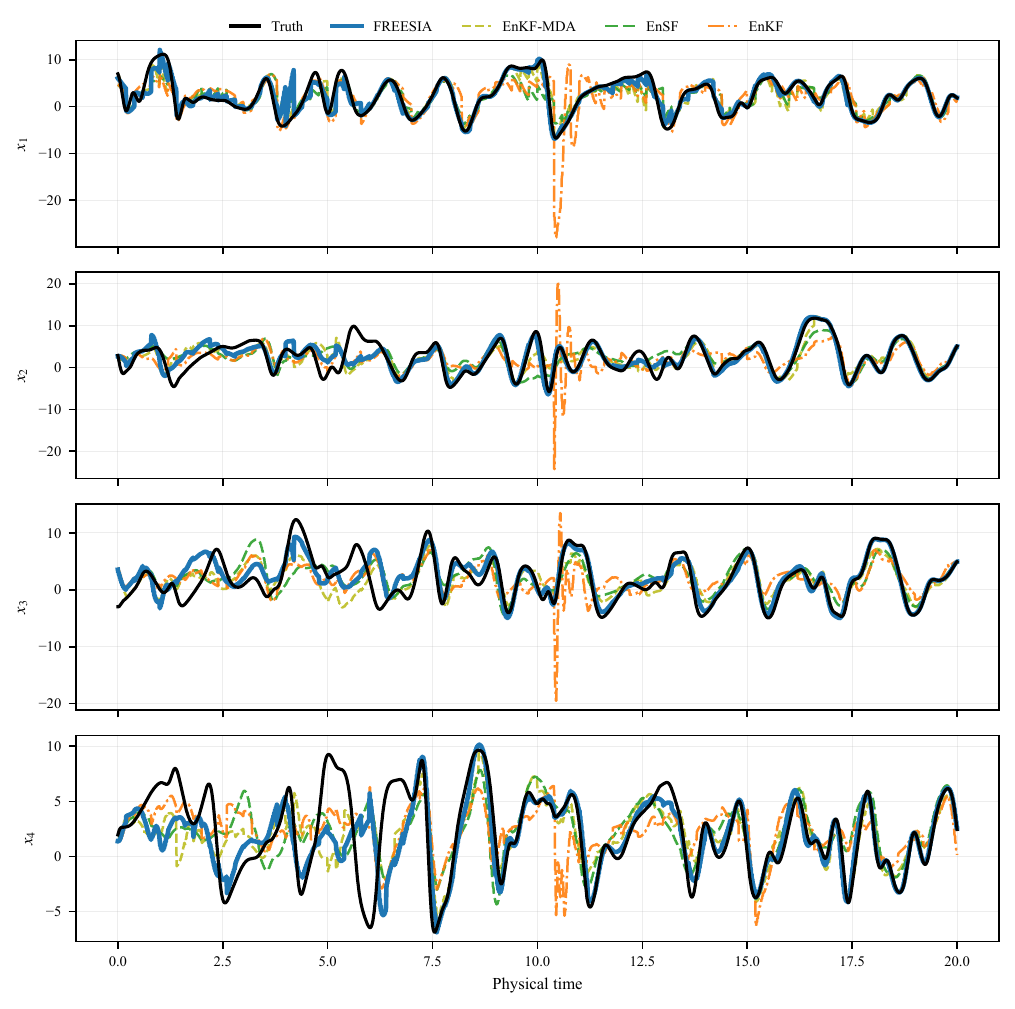}
\caption{Representative trajectories for the first four Lorenz--96 state coordinates under stride-4 arctangent observations. The panels show the truth and analysis trajectories for the methods selected by the corresponding RMSE comparison.}
\label{fig:l96_stride4_traj}
\end{figure}
\FloatBarrier

\paragraph{Stride-4 ablation.}
We repeat the stride-4 experiment with two targeted variants. The first replaces
the localized forecast covariance by its diagonal, removing cross-coordinate
covariances while retaining marginal variances. The second replaces the nonlinear
likelihood-ratio weights by uniform endpoint weights. Across the same 10
initializations and 101 analyses, the time-averaged RMSE increases from $1.575$
for FREESIA to $2.352$ with diagonal covariance and $2.969$ without likelihood
correction. The diagonal variant mainly degrades the unobserved coordinates
($2.702$ versus $1.805$), while the no-correction variant decreases initially
and then accumulates error, reaching final-time RMSE $3.728$ versus $1.382$ for
FREESIA. These changes separately support the roles of cross-covariance transfer
and nonlinear proposal correction in repeated sparse assimilation.

\begin{figure}[!htbp]
\centering
\includegraphics[width=0.88\linewidth]{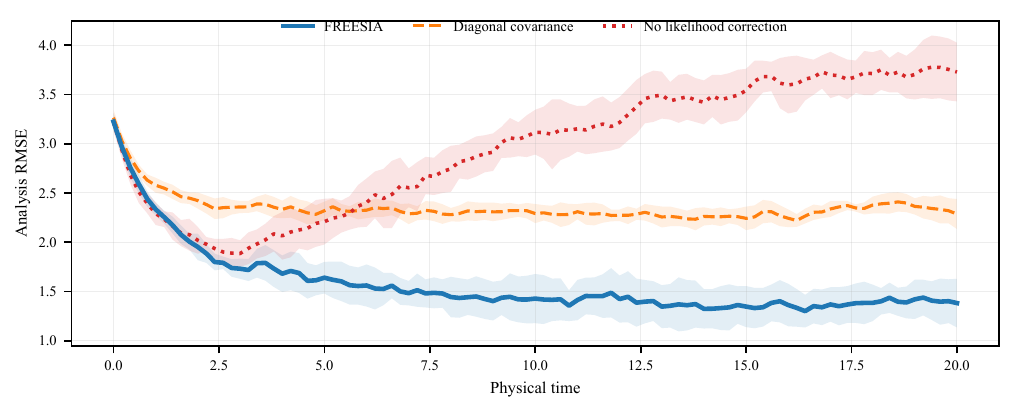}
\caption{Lorenz--96 stride-4 ablation. Curves show mean analysis RMSE over 10 matched runs, and shading denotes one standard deviation across runs. Diagonal covariance removes cross-coordinate covariance terms. The no-likelihood-correction variant replaces the nonlinear likelihood-ratio weights with uniform endpoint weights.}
\label{fig:l96-stride4-ablation}
\end{figure}

\paragraph{Noise robustness.}
As a robustness check, we repeat the stride-4 experiment with observation-noise
standard deviation $\sigma=0.05$. FREESIA remains the lowest-RMSE method,
showing that the sparse-transfer advantage is not specific to the low-noise
setting. The performance gaps among several baselines become smaller, and some
methods improve relative to the $\sigma=0.01$ regime, suggesting that the
sharper likelihood in the high-SNR experiment contributes to the difficulty of
the assimilation problem.
\begin{table}[!htbp]
\centering
\scriptsize
\renewcommand{\arraystretch}{1.08}
\setlength{\tabcolsep}{2.1pt}
\caption{Time-averaged total, observed-coordinate, and unobserved-coordinate RMSE and CRPS for Lorenz--96 stride-4 observations under two noise levels. Lower is better.}
\label{tab:l96-noise-robustness}
\begin{tabular}{lcccccccc}
\toprule
& \multicolumn{4}{c}{$\sigma=0.01$} & \multicolumn{4}{c}{$\sigma=0.05$} \\
\cmidrule(lr){2-5}\cmidrule(lr){6-9}
Method & RMSE & Obs. & Unobs. & CRPS & RMSE & Obs. & Unobs. & CRPS \\
\midrule
\method{} & \textbf{1.575} & 0.373 & \textbf{1.805} & \textbf{0.584} & \textbf{2.076} & 0.938 & \textbf{2.335} & \textbf{1.013} \\
EnKF-MDA & 2.056 & \textbf{0.366} & 2.358 & 0.955 & 2.544 & \textbf{0.894} & 2.891 & 1.342 \\
EnKF & 2.684 & 1.689 & 2.928 & 1.356 & 2.642 & 1.603 & 2.903 & 1.407 \\
EnSF & 2.472 & 1.143 & 2.777 & 1.242 & 2.865 & 1.937 & 3.113 & 1.582 \\
EnFF-F2P & 3.960 & 4.412 & 3.796 & 2.243 & 2.893 & 2.263 & 3.074 & 1.565 \\
LETKF & 3.185 & 1.853 & 3.513 & 1.648 & 3.069 & 1.756 & 3.391 & 1.654 \\
EnFF-OT & 3.749 & 2.637 & 4.051 & 2.636 & 4.751 & 4.573 & 4.807 & 3.447 \\
\bottomrule
\end{tabular}
\end{table}

\paragraph{Accuracy--cost trade-off.}
Figure~\ref{fig:l96-accuracy-cost} fixes $N_{\mathrm{IS}}=5$ and varies only
$T\in\{5,10,20,50\}$, while baseline points use their reported
configurations. Increasing the number of flow steps from 5 to 20 progressively
reduces filtering error, whereas further increasing the flow resolution to 50
yields negligible additional RMSE reduction despite substantially higher
analysis cost. The default $T=20$ setting therefore lies near the empirical
accuracy--cost knee. At a comparable analysis time, FREESIA substantially
improves over EnSF, providing evidence that the observed accuracy gain is not
solely attributable to a larger computational budget.
\begin{figure}[!htbp]
\centering
\includegraphics[width=0.72\linewidth]{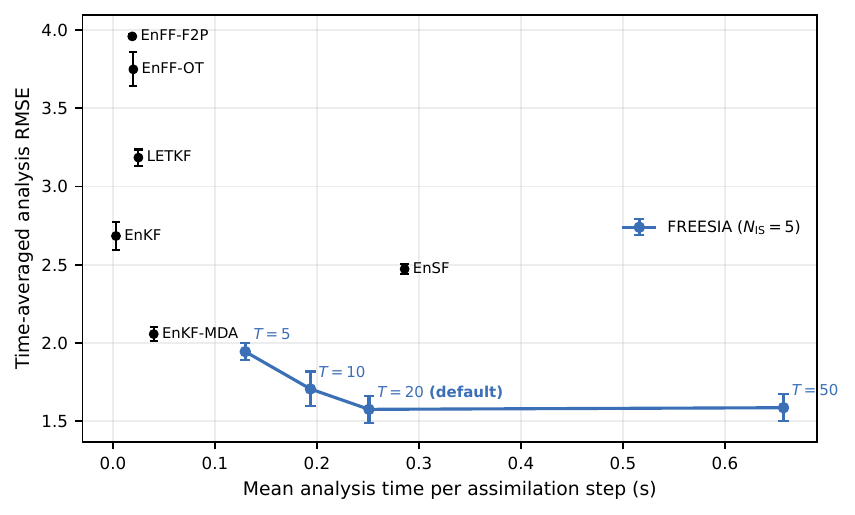}
\caption{Accuracy--cost trade-off for Lorenz--96 stride-4 assimilation. FREESIA fixes $N_{\mathrm{IS}}=5$ and varies only the number of flow steps $T$; error bars denote one standard deviation across repeated runs. Baseline points use the configurations reported in the main comparison.}
\label{fig:l96-accuracy-cost}
\end{figure}

\paragraph{Importance-weight diagnostics.}
For normalized endpoint weights $w^{(s)}$, define
\[
\operatorname{ESS}=\frac{1}{\sum_{s=1}^{N_{\mathrm{IS}}}(w^{(s)})^2},
\qquad
\operatorname{nESS}=\frac{\operatorname{ESS}}{N_{\mathrm{IS}}}
=\frac{1}{N_{\mathrm{IS}}\sum_s(w^{(s)})^2},
\]
so $1/N_{\mathrm{IS}}\le\operatorname{nESS}\le1$.
The higher-resolution diagnostic uses $T=50$ and $N_{\mathrm{IS}}=20$, a finer
flow discretization and more endpoint proposals, only to resolve the
near-endpoint behavior of the importance weights; the default configurations
used for the main accuracy comparisons are unchanged. In the stride-4
small-noise regime, importance weights are strongly concentrated over the early
and intermediate portions of the flow. The concentration decreases rapidly
near $t=1$, consistent with the endpoint contraction analysis in
Appendix~\ref{app:is-control}. Across the filtering sequence, we do not observe
systematic deterioration of the weight behavior at fixed flow times. These
diagnostics are qualitative and do not verify the accumulated-SNIS condition
in Eq.~\ref{eq:accumulated-is-constant}.
\begin{figure}[!htbp]
\centering
\includegraphics[width=0.90\linewidth]{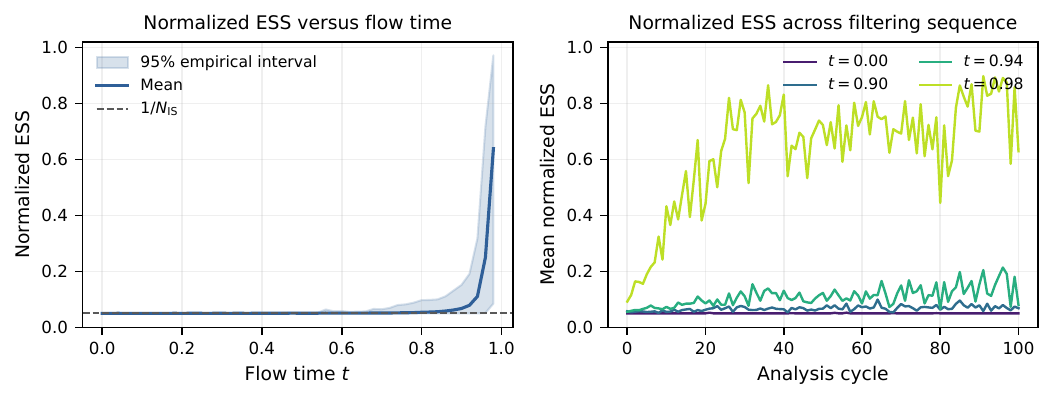}
\caption{Importance-weight diagnostics for the Lorenz--96 stride-4 experiment. Left: normalized effective sample size (nESS) as a function of flow time, aggregated over filtering states and flow particles. Right: nESS across the filtering sequence at selected flow times. Importance weights are strongly concentrated over the early and intermediate portions of the flow, while nESS rises rapidly near the endpoint as the conditional bridge contracts. The concentration is primarily flow-time dependent rather than showing systematic deterioration across the filtering sequence. The higher-resolution configuration is used only to resolve the near-endpoint weight behavior.}
\label{fig:l96-ness}
\end{figure}
\FloatBarrier

\subsection{Kolmogorov Flow}
\label{app:kolmogorov}

The high-dimensional benchmark uses the periodically forced incompressible
Navier--Stokes equations \citep{temam2001navierstokes},
\[
\partial_t u+(u\cdot\nabla)u
=
-\nabla p+\mathrm{Re}^{-1}\Delta u-\alpha u+f,
\qquad
\nabla\cdot u=0,
\]
on a periodic \(128\times128\) grid over \([0,2\pi]^2\). The two velocity
components give \(d=32768\), and the Reynolds number is
\(\mathrm{Re}=1000\). The linear damping coefficient is \(\alpha=0.1\), and
the velocity-form Kolmogorov forcing is
\(f(x,y)=(\sin(4y),0)^\top\), with unit amplitude and forcing wavenumber four.
Equivalently, the forcing contributes \(-4\cos(4y)\) to the vorticity equation.
The numerical time step is determined automatically from a Courant stability
condition.

Before assimilation, the flow is integrated for 10 physical time units so
that both truth trajectories and initial forecast ensembles are initialized
from the developed-flow regime. The spin-up interval is excluded from the
reported filtering statistics.

Both observation regimes use an ensemble size of 20, an observation interval
of approximately 0.2, 50 forecast--analysis intervals, and ten independent
trajectories. The arctangent
experiment observes 328 randomly selected velocity entries
(approximately 1\%), whereas the absolute-arctangent experiment observes 656
entries (approximately 2\%). All compared methods use the same truth
trajectories and observation locations. For the selected
velocity entries, the observation noise is independent Gaussian noise,
\(\epsilon_k\sim\mathcal N(0,0.05^2I)\).

\method{} uses 20 flow steps, \(N_{\mathrm{IS}}=5\) endpoint importance samples,
\(\gamma=0.2\), localization radius \(r_{\mathrm{loc}}=16\), and covariance
inflation 1.1 in both reported observation regimes. The multi-endpoint
construction permits nonuniform likelihood-ratio weights while retaining the
covariance-aware observation-adaptive proposal. The matched Lorenz--96 stride-4
ablation in Appendix~B.2 isolates the contribution of nonlinear likelihood
correction.

Table~\ref{tab:kf_loc} reports the covariance settings used by \method{} and
the Kalman-type baselines.
All seven methods complete the ten repeated assimilation sequences in both
observation regimes.

\paragraph{Time-resolved RMSE.}
Figure~\ref{fig:2dkf-rmse} reports the analysis trajectories for the available
methods. Each panel retains the four methods with the lowest time-averaged RMSE
in that regime.

\begin{figure}[!htbp]
\centering
\includegraphics[width=0.82\linewidth]{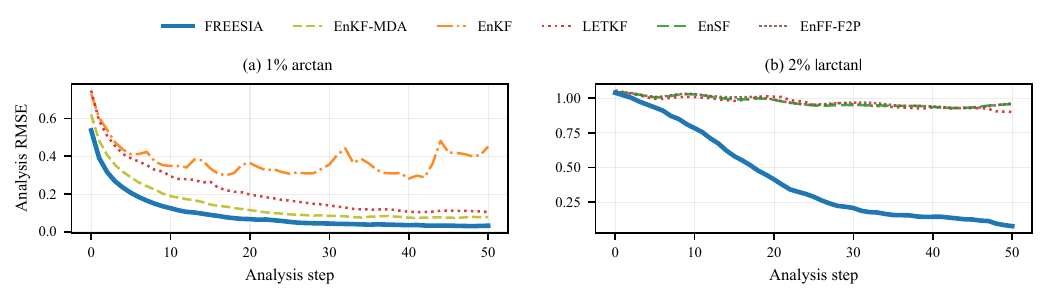}
\caption{Analysis RMSE for 1\% arctangent observations (left) and 2\% absolute-arctangent observations (right). Each panel shows the four methods with the lowest time-averaged RMSE across the ten trajectories.}
\label{fig:2dkf-rmse}
\end{figure}

\paragraph{Accuracy--cost trade-off.}
Figure~\ref{fig:kf-accuracy-cost} reports an accuracy--cost sweep for the 1\%
arctangent Kolmogorov-flow regime. The sweep fixes
$N_{\mathrm{IS}}=1$, $r_{\mathrm{loc}}=16$, and the remaining \method{}
hyperparameters, and varies only the number of flow steps
$T\in\{5,10,20\}$. The default benchmark configuration
$(T=20,N_{\mathrm{IS}}=5)$, which uses the same localization radius, is
included as the rightmost reference point.
Filtering error decreases rapidly over the low-to-moderate budget range and
then saturates; the default configuration incurs additional cost for
comparatively modest further improvement in this monotone-observation regime.
\begin{figure}[!htbp]
\centering
\includegraphics[width=0.72\linewidth]{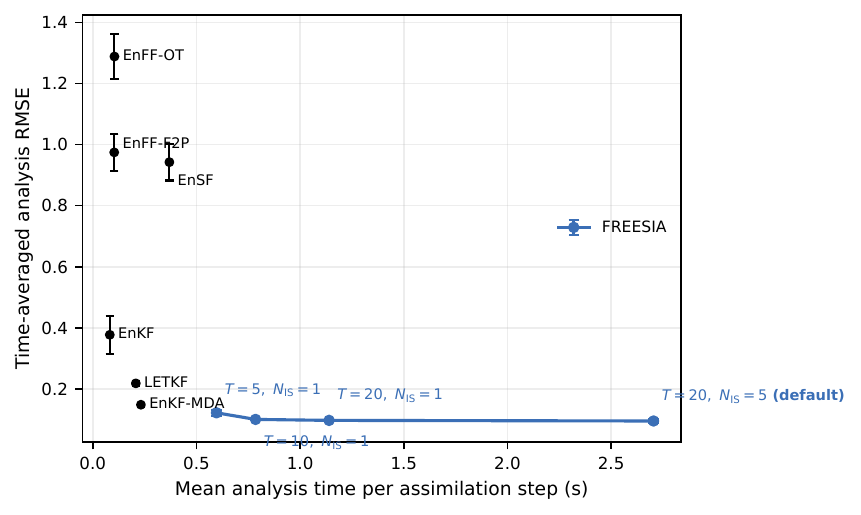}
\caption{Accuracy--cost trade-off on 2-D Kolmogorov flow with 1\% random arctangent observations. The sweep fixes $N_{\mathrm{IS}}=1$ and $r_{\mathrm{loc}}=16$, and varies only the number of flow steps $T$. The rightmost reference point is the default benchmark configuration $(T=20,N_{\mathrm{IS}}=5)$, with the same localization radius. Error bars denote one standard deviation across ten trajectories. Baseline points use their reported configurations.}
\label{fig:kf-accuracy-cost}
\end{figure}

\begin{table}[!htbp]
\centering
\scriptsize
\caption{Localization radius \(r_{\mathrm{loc}}\) and covariance inflation \(\lambda_{\mathrm{infl}}\) for the Kolmogorov-flow comparisons.}
\label{tab:kf_loc}
\begin{tabular}{lrrrr}
\toprule
& \multicolumn{2}{c}{1\% atan} & \multicolumn{2}{c}{2\% abs-atan} \\
\cmidrule(lr){2-3}\cmidrule(lr){4-5}
Method & \(r_{\mathrm{loc}}\) & \(\lambda_{\mathrm{infl}}\) & \(r_{\mathrm{loc}}\) & \(\lambda_{\mathrm{infl}}\) \\
\midrule
\method{} & 16 & 1.1 & 16 & 1.1 \\
EnKF & 8 & 1.1 & 6 & 1.1 \\
EnKF-MDA & 10 & 1.2 & 6 & 1.1 \\
LETKF & 3 & 1.1 & 2 & 1.1 \\
\bottomrule
\end{tabular}
\end{table}
\FloatBarrier

\paragraph{Statistical uncertainty.}
Tables~\ref{tab:dw-uncertainty}--\ref{tab:kf-uncertainty} report means and
standard deviations of per-run time-averaged metrics. Standard deviations
describe variability across independent runs and are not confidence
intervals.

\begin{table}[!htbp]
\centering
\scriptsize
\setlength{\tabcolsep}{3.5pt}
\caption{Double-Well metrics (mean $\pm$ standard deviation over 10 runs).}
\label{tab:dw-uncertainty}
\begin{tabular}{lccccc}
\toprule
Method & JS & $W_1$ & Mode error & RMSE & CRPS \\
\midrule
\method{} & $0.03134\pm0.00047$ & $0.05936\pm0.00053$ & $0.03503\pm0.00009$ & $0.14417\pm0.00035$ & $0.09915\pm0.00002$ \\
EnKF-MDA & $0.18521\pm0.00001$ & $0.31992\pm0.00004$ & $0.25878\pm0.00001$ & $0.43262\pm0.00003$ & $0.37445\pm0.00001$ \\
EnKF & $0.18365\pm0.00035$ & $0.23281\pm0.00050$ & $0.15807\pm0.00010$ & $0.32489\pm0.00051$ & $0.25673\pm0.00017$ \\
EnSF & $0.31497\pm0.00161$ & $0.66648\pm0.00404$ & $0.35924\pm0.00398$ & $0.69883\pm0.00261$ & $0.37946\pm0.00186$ \\
EnFF-OT & $0.54998\pm0.00007$ & $0.20219\pm0.00003$ & $0.12996\pm0.00001$ & $0.26335\pm0.00004$ & $0.25411\pm0.00004$ \\
EnFF-F2P & $0.35776\pm0.00115$ & $0.15151\pm0.00243$ & $0.08547\pm0.00072$ & $0.23215\pm0.00183$ & $0.16993\pm0.00080$ \\
\bottomrule
\end{tabular}
\end{table}

\begin{table}[!htbp]
\centering
\scriptsize
\setlength{\tabcolsep}{1.7pt}
\caption{Lorenz--96 RMSE and CRPS (mean $\pm$ standard deviation over 10 runs). For the dense observation regime, all coordinates are observed, so only total RMSE and CRPS are reported. For the sparse regimes, total, observed-coordinate, and unobserved-coordinate RMSE are reported together with CRPS.}
\label{tab:l96-uncertainty}
\resizebox{\linewidth}{!}{%
\begin{tabular}{lcccccccccc}
\toprule
& \multicolumn{2}{c}{Dense abs-atan} & \multicolumn{4}{c}{Stride-2 atan} & \multicolumn{4}{c}{Stride-4 atan} \\
\cmidrule(lr){2-3}\cmidrule(lr){4-7}\cmidrule(lr){8-11}
Method & RMSE & CRPS & RMSE & Obs. & Unobs. & CRPS & RMSE & Obs. & Unobs. & CRPS \\
\midrule
\method{} & $0.490\pm0.038$ & $0.174\pm0.007$ & $0.615\pm0.060$ & $0.417\pm0.025$ & $0.757\pm0.082$ & $0.241\pm0.014$ & $1.575\pm0.085$ & $0.373\pm0.017$ & $1.805\pm0.098$ & $0.584\pm0.041$ \\
EnKF-MDA & $1.384\pm0.255$ & $0.586\pm0.107$ & $0.987\pm0.243$ & $0.673\pm0.153$ & $1.211\pm0.307$ & $0.232\pm0.006$ & $2.056\pm0.045$ & $0.366\pm0.021$ & $2.358\pm0.051$ & $0.955\pm0.024$ \\
EnKF & $3.053\pm0.155$ & $1.712\pm0.148$ & $1.380\pm0.153$ & $1.120\pm0.141$ & $1.573\pm0.168$ & $0.422\pm0.026$ & $2.684\pm0.089$ & $1.689\pm0.130$ & $2.928\pm0.078$ & $1.356\pm0.020$ \\
LETKF & $6.161\pm1.161$ & $3.921\pm0.329$ & $1.533\pm0.225$ & $1.243\pm0.192$ & $1.753\pm0.253$ & $0.494\pm0.051$ & $3.185\pm0.052$ & $1.853\pm0.069$ & $3.513\pm0.052$ & $1.648\pm0.022$ \\
EnSF & $3.901\pm0.012$ & $2.340\pm0.008$ & $2.027\pm0.051$ & $1.531\pm0.036$ & $2.422\pm0.064$ & $0.976\pm0.035$ & $2.472\pm0.031$ & $1.143\pm0.017$ & $2.777\pm0.035$ & $1.242\pm0.020$ \\
EnFF-OT & $4.341\pm0.206$ & $2.959\pm0.213$ & $2.817\pm0.085$ & $2.303\pm0.069$ & $3.247\pm0.100$ & $1.825\pm0.076$ & $3.749\pm0.109$ & $2.637\pm0.073$ & $4.051\pm0.119$ & $2.636\pm0.101$ \\
EnFF-F2P & $1.933\pm0.054$ & $0.721\pm0.029$ & $1.486\pm0.042$ & $1.045\pm0.029$ & $1.821\pm0.053$ & $0.609\pm0.024$ & $3.960\pm0.013$ & $4.412\pm0.012$ & $3.796\pm0.017$ & $2.243\pm0.008$ \\
\bottomrule
\end{tabular}
}
\end{table}

\begin{table}[!htbp]
\centering
\scriptsize
\setlength{\tabcolsep}{1.7pt}
\caption{Kolmogorov-flow total, observed-coordinate, and unobserved-coordinate RMSE and CRPS (mean $\pm$ standard deviation over ten trajectories).}
\label{tab:kf-uncertainty}
\resizebox{\linewidth}{!}{%
\begin{tabular}{lcccccccc}
\toprule
& \multicolumn{4}{c}{1\% atan} & \multicolumn{4}{c}{2\% abs-atan} \\
\cmidrule(lr){2-5}\cmidrule(lr){6-9}
Method & All & Obs. & Unobs. & CRPS & All & Obs. & Unobs. & CRPS \\
\midrule
\method{} & $0.095\pm0.005$ & $0.076\pm0.007$ & $0.095\pm0.005$ & $0.052\pm0.003$ & $0.422\pm0.101$ & $0.410\pm0.105$ & $0.422\pm0.101$ & $0.209\pm0.036$ \\
EnKF-MDA & $0.148\pm0.007$ & $0.115\pm0.012$ & $0.148\pm0.007$ & $0.081\pm0.005$ & $1.121\pm0.088$ & $1.115\pm0.105$ & $1.121\pm0.088$ & $0.797\pm0.070$ \\
EnKF & $0.377\pm0.063$ & $0.372\pm0.070$ & $0.377\pm0.063$ & $0.200\pm0.051$ & $0.987\pm0.123$ & $0.984\pm0.123$ & $0.987\pm0.123$ & $0.681\pm0.100$ \\
LETKF & $0.218\pm0.008$ & $0.177\pm0.012$ & $0.218\pm0.008$ & $0.117\pm0.006$ & $0.970\pm0.062$ & $1.088\pm0.072$ & $0.967\pm0.062$ & $0.572\pm0.031$ \\
EnSF & $0.942\pm0.060$ & $0.568\pm0.067$ & $0.945\pm0.060$ & $0.547\pm0.035$ & $0.974\pm0.053$ & $0.967\pm0.085$ & $0.975\pm0.052$ & $0.569\pm0.030$ \\
EnFF-OT & $1.288\pm0.072$ & $1.148\pm0.081$ & $1.290\pm0.072$ & $1.024\pm0.060$ & $1.342\pm0.041$ & $1.342\pm0.052$ & $1.342\pm0.041$ & $1.069\pm0.030$ \\
EnFF-F2P & $0.974\pm0.060$ & $0.776\pm0.081$ & $0.976\pm0.059$ & $0.567\pm0.034$ & $0.980\pm0.051$ & $0.977\pm0.069$ & $0.980\pm0.051$ & $0.570\pm0.028$ \\
\bottomrule
\end{tabular}
}
\end{table}
\FloatBarrier

\paragraph{Runtime.}
Table~\ref{tab:runtime} reports mean analysis time per assimilation step for
the reported Lorenz--96 and Kolmogorov-flow configurations. The Kolmogorov-flow
entries use the main \(N_{\mathrm{IS}}=5\) configuration. Timings are
measured on a laptop equipped with an NVIDIA GeForce RTX 5070 Laptop GPU.

\begin{table}[!htbp]
\centering
\small
\setlength{\tabcolsep}{5pt}
\caption{Mean analysis time per assimilation step in seconds.}
\label{tab:runtime}
\begin{tabular}{lccccc}
\toprule
Method & L96 dense & L96 stride-2 & L96 stride-4 & KF $\arctan$ & KF $|\arctan|$ \\
\midrule
\textbf{\method{}} (ours) & 1.377 & 0.478 & 0.251 & 2.704 & 3.397 \\
EnKF-MDA & 0.216 & 0.083 & 0.040 & 0.232 & 0.394 \\
EnKF & 0.012 & 0.005 & 0.003 & 0.082 & 0.194 \\
LETKF & 0.009 & 0.021 & 0.025 & 0.208 & 0.200 \\
EnSF & 0.343 & 0.301 & 0.286 & 0.369 & 0.385 \\
EnFF-OT & 0.021 & 0.022 & 0.020 & 0.105 & 0.102 \\
EnFF-F2P & 0.020 & 0.020 & 0.019 & 0.103 & 0.101 \\
\bottomrule
\end{tabular}
\end{table}

\FloatBarrier

\subsection{Computational Complexity and High-Dimensional Scaling}
\label{app:highdim-scaling}

We further examine computational scaling using fully observed Lorenz--96,
\[
y_i=\arctan(x_i)+\epsilon_i,\qquad
\epsilon_i\sim\mathcal N(0,0.05^2),
\]
with $d\in\{10^4,10^5,10^6\}$ and ensemble size $M=20$. For each dimension,
three independent runs use matched initializations and observation-noise
sequences across methods, with 51 analysis times per run. We compare EnSF and
EnFF-F2P with \textbf{FREESIA-Diag}, a lightweight specialization of FREESIA
for extremely high-dimensional fully observed settings. FREESIA-Diag uses a
diagonal forecast covariance, $\gamma=0$, $N_{\mathrm{IS}}=1$, and $T=20$;
EnFF-F2P uses 20 flow points, and EnSF uses 50 reverse-diffusion steps.

\begin{table}[!htbp]
\centering
\small
\caption{Leading analysis complexity of the scaling-study implementations.
$M$, $T$, $d$, and $m$ denote ensemble size, generative steps, state dimension,
and observation dimension, respectively. For structured FREESIA,
$C_{\mathrm{cond}}$, $C_{\mathrm{end}}$, and $C_{\mathrm{like}}$ are defined in
Appendix~\ref{app:covariance-implementation}; $S_{\Sigma}$ denotes
structured-covariance storage.}
\label{tab:scaling_complexity}
\begin{tabular}{lcc}
\toprule
Method & Computation & Memory \\
\midrule
EnSF (diagonal score) & $O(MTd)$ & $O(Md)$ \\
EnFF-F2P & $O(M^2Td)$ & $O(Md+M^2)$ \\
FREESIA (structured) & \shortstack[c]{$O(TM^2C_{\mathrm{cond}}+{}$\\$TMN_{\mathrm{IS}}(C_{\mathrm{end}}+C_{\mathrm{like}}))$} & $O(Md+S_{\Sigma}+m^2)$ \\
FREESIA-Diag & $O(MTd)$ & $O(Md)$ \\
\bottomrule
\end{tabular}
\end{table}

Table~\ref{tab:scaling_complexity} separates the general structured
implementation from the diagonal specialization used in this scaling study.
As detailed in Appendix~\ref{app:covariance-implementation}, general FREESIA
retains explicit dependence on observation dimension through the $m\times m$
innovation covariance $S_{t,j}$. Under structured covariance actions,
row-sparse observations, and a direct dense observation-space solve,
$C_{\mathrm{cond}}=O(dm+m^3)$, $C_{\mathrm{end}}=O(d+m^2)$, and
$C_{\mathrm{like}}=O(m)$ for diagonal $R$. In FREESIA-Diag, the forecast
covariance and component-wise observation Jacobian
$H=\operatorname{Diag}(h'_1,\ldots,h'_d)$ are diagonal, as is $R$; hence
$S_{t,j}=H\Sigma_{1\mid t}H^{\top}+R$ is diagonal and its solve reduces to
coordinatewise scalar operations. Together with $\gamma=0$, which makes the
mixture components coincide, this gives $O(MTd)$ computation and $O(Md)$
memory for fixed $N_{\mathrm{IS}}$, even when $m=d$.

\begin{figure}[!htbp]
\centering
\includegraphics[width=\linewidth]{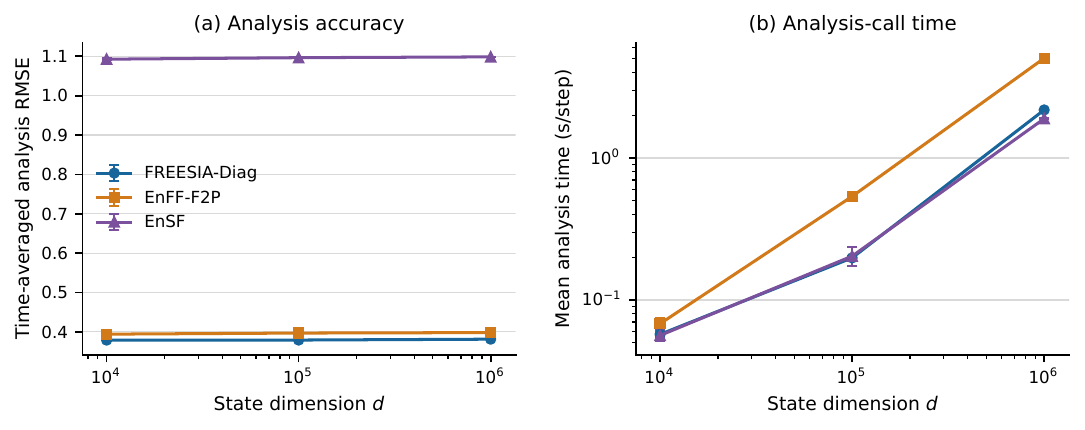}
\caption{Fully observed Lorenz--96 scaling with $M=20$. Left: time-averaged
analysis RMSE. Right: mean analysis time per assimilation step. Each point
averages 51 analyses over three independent runs; error bars denote one
standard deviation across run means.}
\label{fig:l96-highdim-scaling}
\end{figure}

Figure~\ref{fig:l96-highdim-scaling} reports the corresponding accuracy and
analysis-time scaling. All three implementations scale to $d=10^6$ in the
tested configuration. FREESIA-Diag maintains a nearly dimension-independent
time-averaged RMSE of 0.379--0.382; at $d=10^6$, the RMSEs are 0.382, 0.398,
and 1.098 for FREESIA-Diag, EnFF-F2P, and EnSF, respectively. Its analysis
time grows approximately linearly with state dimension over the tested range,
consistent with the $O(MTd)$ complexity in Table~\ref{tab:scaling_complexity}.
These results show that the diagonal specialization preserves useful filtering
accuracy while extending FREESIA to million-dimensional state spaces with
linear state-dimensional storage and computation.
\FloatBarrier

\end{document}